\documentclass[10pt,twocolumn,letterpaper]{article}

\usepackage{times}
\usepackage{enumitem}
\usepackage{algorithm}
\usepackage{algpseudocode}

\makeatletter
\font\cvprelvbf=ptmb scaled 1100
\font\cvprtenbf=ptmb scaled 1000
\renewcommand\section{\@startsection{section}{1}{\z@}%
  {-10pt plus -2pt minus -2pt}{7pt}{\large\bfseries}}
\renewcommand\subsection{\@startsection{subsection}{2}{\z@}%
  {-8pt plus -2pt minus -2pt}{5pt}{\cvprelvbf}}
\renewcommand\subsubsection{\@startsection{subsubsection}{3}{\z@}%
  {-6pt plus -2pt minus -2pt}{3pt}{\cvprtenbf}}
\makeatother
\usepackage{graphicx}
\usepackage{booktabs}
\usepackage{amsmath,amssymb,amsthm}
\usepackage{subcaption}
\usepackage{arydshln}
\usepackage{cancel}
\usepackage{xcolor}
\usepackage[numbers,sort&compress]{natbib}
\usepackage{titlesec}
\titlespacing*{\paragraph}{0pt}{8pt}{0.5em}

\newtheorem{proposition}{Proposition}
\newtheorem{definition}{Definition}

\newtheorem{observation}{Observation}

\newtheorem{corollary}{Corollary}

\definecolor{linkblue}{rgb}{0.21,0.49,0.74}
\definecolor{airforceblue}{rgb}{0.36,0.54,0.66}

\usepackage[pagebackref,breaklinks,colorlinks,allcolors=linkblue]{hyperref}

\title{Coordination on a Budget: Federated Active Learning with Few Labels}
\author{
Liam Mohr and Daphna Weinshall\\
School of Computer Science and Engineering \\
The Hebrew University of Jerusalem, 
Jerusalem 91904, Israel\\
\texttt{\{liam.mohr,daphna\}@mail.huji.ac.il}
}
\date{}

\begin{document}

\maketitle

\begin{center}
    \textbf{Abstract}
\end{center}
\vspace{-0.5em}

{\small\itshape


Federated Active Learning (FAL) addresses the dual challenges of data privacy and label scarcity, where the absence of a global data view introduces additional hurdles for coordinated query selection. We study cross-silo FAL in the low-budget regime, where annotation decisions are most critical. We characterize, both theoretically and empirically, a \emph{heterogeneity reversal}: in low-budget settings, homogeneous (IID) data requires stronger coordination to avoid redundant queries, whereas heterogeneous data naturally promotes diversity; this trend reverses at higher budgets. Thus, in contrast to the standard federated learning (FL) narrative where heterogeneity is a primary challenge, we show that IID settings are more challenging for query selection in FAL.

Motivated by these findings, we propose a new FAL framework that utilizes federated representation learning to align client data in a shared embedding space. This enables the server to perform globally coordinated active selection over optionally obfuscated client embeddings, while annotation remains local to each client. Although our framework operates in the more challenging low-budget regime, it achieves performance that surpasses existing FAL methods even when they are given substantially larger annotation budgets, demonstrating the value of centralized coordination under privacy constraints.

\par}

\section{Introduction}

Modern machine learning systems are increasingly deployed in settings where data is decentralized and subject to strict privacy and governance constraints. In domains such as healthcare, finance, and telecommunications, data is distributed across institutions and cannot be shared due to regulatory or operational limitations. Federated Learning (FL) addresses this challenge by enabling collaborative model training without access to raw data. However, FL typically assumes labeled data, while in practice, unlabeled data is abundant and annotations are costly and scarce.

Active Learning (AL) offers a complementary solution by selectively querying the most informative samples for annotation. Yet, classical AL relies on centralized access to the unlabeled pool, allowing the model to compare candidates globally. In federated settings, this assumption breaks down: data remains distributed across clients, and coordination is restricted by privacy and communication constraints. This gives rise to the setting of \emph{Federated Active Learning} (FAL), where query selection must be performed without direct access to a global data view.

A central challenge in FAL is coordinating selection across clients to avoid redundant or suboptimal queries. While data heterogeneity is traditionally viewed as a primary obstacle in federated learning, we uncover a contrasting phenomenon in the low-budget regime. Specifically, we observe a \emph{heterogeneity reversal}: when clients have similar (IID) data distributions, independent selection leads to redundant queries and poor global coverage, making coordination essential. In contrast, heterogeneous (non-IID) data naturally promotes diversity in selected samples, reducing the need for coordination. This trend reverses at higher budgets, where non-IID data introduces bias in uncertainty estimation, reverting to the classical challenges of FL.
We formalize this interaction in Section~\ref{sec:hetero}, showing that the value of coordination depends jointly on data heterogeneity and labeling budget. Table~\ref{tab:coordination_utility} summarizes the resulting reversal between diversity- and uncertainty-driven regimes.

\begin{figure*}[t]
    \centering
\includegraphics[width=\linewidth]{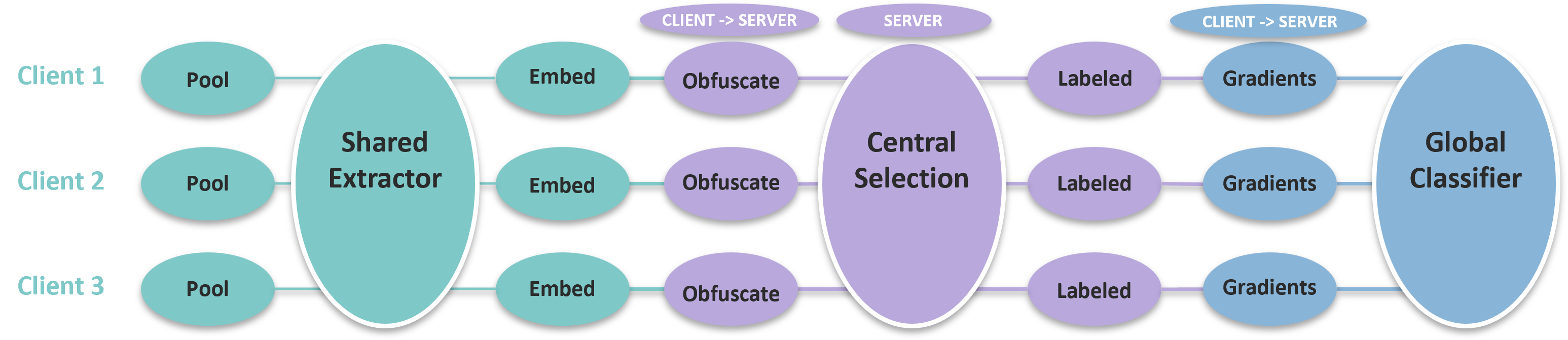}
    \caption{Overview of the proposed federated active learning pipeline. A federated feature extractor first induces a shared embedding space across clients. Each client then computes local embeddings and may apply an obfuscation mechanism before transmitting them to the server. The server performs centralized active selection in the aggregated embedding space, while the selected samples are labeled locally by the corresponding clients and used for subsequent federated training.}
    \label{fig:method_visualization}
\end{figure*}

\begin{table*}[t]
\small
\centering
\begin{tabular}{@{}llll@{}}
\toprule
\textbf{Criterion} & \textbf{IID Setting} & \textbf{Non-IID Setting} & \textbf{novelty}\\ \midrule
\textbf{Diversity} & \textbf{High Coordination Required} & \textbf{Low Coordination Required} & new challenge\\
& (Avoids overlap/redundancy) & (Disjoint supports ensure variety) &\\ \addlinespace
\textbf{Uncertainty} & \textbf{Low Coordination Required} & \textbf{High Coordination Required} & old challenge\\
& (Local $\approx$ Global informative) & (Corrects for local representation bias) &\\ \bottomrule
\end{tabular}
\caption{
\textbf{Coordination vs. heterogeneity.} In the low-budget regime, IID data require coordination to ensure diversity, making them more challenging than non-IID data - a reversal of the standard FL setting.}
\label{tab:coordination_utility}
\end{table*}

Our analysis reveals a key limitation of existing FAL approaches: they do not explicitly account for the interaction between data distribution and labeling budget. A natural baseline is to apply active learning independently at each client and rely on federated training only after annotation, as in prior low-budget FAL pipelines~\citep{ono2025exploring}. However, this confines inter-client coordination to the model-training stage, leaving query selection decoupled across clients and failing to fully exploit the potential synergy between active learning and federated learning. 

We address this limitation through globally coordinated query selection in a shared federated representation space. A federatively learned feature extractor aligns client embeddings, allowing the server to compare candidates across clients and coordinate selection while raw data remain local. Selected samples are then labeled locally and used for downstream federated training (see Figure~\ref{fig:method_visualization}).

While this design enables effective coordination across clients, it also introduces a potential privacy risk, as shared embeddings may leak information about the underlying data; in visual domains, feature representations can be vulnerable to inversion or reconstruction attacks \citep{chatzikokolakis2013broadening
}. To address this challenge, we investigate two privacy-preserving approaches, using either controlled perturbation of embeddings or centroid-based aggregation, and study the resulting trade-offs between privacy protection and active selection performance.

\paragraph{Our contributions are:}
\begin{itemize}[itemsep=1pt, topsep=3pt, parsep=0pt, partopsep=0pt]
    \item We identify and analyze a \emph{heterogeneity reversal} in federated active learning: in low-budget regimes, IID data leads to redundant sampling and requires stronger coordination than heterogeneous data.
    \item We propose a new framework for \emph{globally coordinated query selection} using shared federated embedding,  with two embedding-obfuscation mechanisms for data protection. 
    \item We show that global coordination substantially improves label efficiency in the low-budget regime, while retaining much of this advantage under embedding obfuscation, yielding a favorable privacy-performance trade-off.
\end{itemize}

\section{Related Work}
\label{sec:related}

\paragraph{Low-Budget Active Learning.}
Early in the AL process, uncertainty-based methods, such as Entropy Sampling \citep{wang2014new}, Least Confidence \citep{lewis1994sequential}, and Margin Sampling \citep{scheffer2001active}, frequently underperform because the model's predictive signal is of low quality \citep{hacohen2022active}. To mitigate this, recent approaches leverage self-supervised representations. For instance, \emph{TypiClust} \citep{hacohen2022active} prioritizes representative samples from high-density clusters. Similarly, \emph{ProbCover} \citep{yehuda2022active} and \emph{MaxHerding} \citep{bae2024maxherding} frame active selection as a probabilistic coverage problem, strategically picking samples to maximize the probability that the unlabeled data manifold is spanned within the given budget constraint.

\paragraph{Federated Learning.}
Federated learning (FL) enables collaborative model training across distributed clients while keeping data localized, primarily addressing privacy, communication efficiency, and statistical heterogeneity 
\citep{mcmahan2017communication, kairouz2021advances}. A large body of work focuses on mitigating the challenges arising from non-IID data distributions and limited communication bandwidth \citep{li2020federated, karimireddy2020scaffold}. 

In contrast, our work considers a small number of clients with relatively homogeneous data distributions, allowing us to isolate a different bottleneck: labeled-data scarcity and its interaction with distributed active selection. This perspective complements existing FL research by highlighting challenges that arise even when communication constraints and data heterogeneity are less pronounced.

\paragraph{Federated contrastive representation learning.}
Methods that are effective for federated learning under supervised objectives, such as cross-entropy, are significantly less effective for contrastive learning, since the global self-supervised objective does not decompose into a sum of local objectives \citep{zhuang2021divergence}. This mismatch can lead to degraded representations when applying standard federated averaging. Prior work has proposed adaptations such as prototype-based alignment or modified contrastive objectives to mitigate this issue \citep{ye2021fedproto, li2021federatedcontrastive}. Much work has been dedicated to address the adversarial effect of non-IID client distribution on FCRL \citep{dong2021federated,zhuang2021divergence,jing2024fedsc,han2022fedx,seo2024relaxed,louizosmutual}. 

\paragraph{Federated Active Learning.}
Most federated active learning (FAL) methods rely on model-based scoring to estimate sample informativeness, making them effective primarily in high-budget regimes while degrading at low budgets. Early work such as \citep{ahmed2020active} adopts a \emph{separate} AL (S-AL) paradigm, where selection is performed locally at each client. In contrast, F-AL \citep{ahn2024federated} enables collaborative evaluation but shows clear gains only at higher budgets (e.g., 150--200 labels per class on CIFAR-100). Others address this mismatch by selecting samples informative for both local and global objectives \citep{kim2022lg, cao2023knowledge, kim2023re}, typically focusing on non-IID settings where heterogeneity complicates uncertainty estimation \citep{zhang2023affectfal}. Finally, in Active Federated Learning (AFL), the decision concerns \emph{which clients} to train rather than which samples to label \citep{goetz2019active}.

Among existing FAL methods, \citep{ono2025exploring} specifically targets and remains effective in the low-budget regime. This approach follows a separate active learning (S-AL) paradigm, where sample evaluation is performed locally at each client. By leveraging a selection criterion tailored to low-budget settings \citep{hacohen2022active}, it performs well in federated scenarios and outperforms existing FAL baselines \citep{ono2025exploring}. We therefore adopt it as a primary baseline in our experiments.

\section{Heterogeneity \& Inter-Client Coordination}
\label{sec:hetero}

Standard active learning balances two criteria: \textbf{diversity}, which promotes coverage of the feature space, and \textbf{uncertainty}, which targets low-confidence regions. In federated settings, the value of inter-client coordination depends critically on data heterogeneity. We analyze this interaction in the context of active learning query selection criteria, and highlight two key effects:

\begin{enumerate}
\item \textbf{Diversity-centric selection:} Under IID data, independent clients tend to select overlapping samples, requiring strong coordination to avoid redundancy. In contrast, heterogeneous data naturally partitions the space, reducing the need for coordination.

\item \textbf{Uncertainty-centric selection:} Under IID data, independent clients tend to learn similar models. In contrast, under non-IID data, local models become biased estimators of global uncertainty, making coordination necessary. This aligns with classical FL results~\cite{mcmahan2017communication}, where heterogeneity induces model divergence.
\end{enumerate}

Together, these effects reveal a \emph{reversed vulnerability} in Federated Active Learning: coordination is most critical for diversity under \emph{IID data}, and for uncertainty under \emph{non-IID data}. We formalize this relationship in Section~\ref{sec:heter-div}, focusing on the diversity regime where this reversal departs from standard FL, and validate it empirically in Sections~\ref{sec:full-pipeline} and~\ref{sec:same_embedding_ablation}, with detailed results provided in Appendix~\ref{sec:ablation}.

\subsection{Notation and Preliminaries}
\label{sec:notation}

Let $\mathcal{X}\subseteq\mathbb{R}^d$ denote the feature space and $\mathcal{Y}$ the label space. Consider $K$ clients, where client $k\in\{1,\ldots,K\}$ possesses a local data distribution $\mathcal{P}_k$ over $\mathcal{X}\times\mathcal{Y}$. Let
\[
\mathcal{X}_k
=
\operatorname{supp}(\mathcal{P}_k^X)
\]
denote the support of the feature marginal of client $k$.

\begin{definition}[Coordination Gap]
Let $\Phi(S)$ denote a selection utility for a set $S\subseteq\mathcal{X}$. 
Let $b_k$ denote the annotation budget of client $k$, with
\[
\sum_{k=1}^K b_k=B.
\]
Define the optimal coordinated utility as
\[
\Phi_B^\star
=
\max_{\substack{
S_k\subseteq\mathcal{X}_k,\ |S_k|=b_k\\
k=1,\ldots,K
}}
\Phi\left(
\bigcup_{k=1}^K S_k
\right).
\]
For each client, let
\[
S_k^\star
\in
\arg\max_{\substack{
S\subseteq\mathcal{X}_k\\
|S|=b_k
}}
\Phi(S)
\]
denote its locally optimal selection.
The coordination gap is
\begin{equation}
\label{eq:gap}
\Delta_K
=
\Phi_B^\star
-
\Phi\left(
\bigcup_{k=1}^K S_k^\star
\right).
\end{equation}
\end{definition}

\subsection{Data Heterogeneity and Diversity}
\label{sec:heter-div}

We show that the utility of coordination increases with overlap between clients' selections and vanishes when heterogeneity separates their accessible regions. We formalize this relationship using a cell-coverage model.

Let the shared embedding space be partitioned into $M$ cells,
$
\mathcal{C}
=
\{C_1,\ldots,C_M\},
$
representing regions such as clusters or high-density neighborhoods. Define the coverage utility
\begin{equation}
\label{eq:cell-coverage}
\Phi(S)
=
\sum_{m=1}^M
\mathbf{1}\{S\cap C_m\neq\emptyset\}.
\end{equation}
For each client, define the accessible cell set
\begin{equation}
\label{eq:accessible-cells}
A_k
=
\{m:C_m\cap\mathcal{X}_k\neq\emptyset\},
\end{equation}
and the cell-selection probability
\begin{equation}
\label{eq:akm}
a_{k,m}
=
\Pr(S_k^\star\cap C_m\neq\emptyset).
\end{equation}
Thus, $a_{k,m}=0$ whenever $m\notin A_k$.
 
We start from the expected coordination gap $\bar{\Delta}_K$:
\begin{eqnarray}
\bar{\Delta}_K
&=&
\mathbb{E}[\Delta_K]
=
\mathbb{E}[\Phi_B^\star]
-
\mathbb{E}[\Phi({\bigcup_{k=1}^K S_k^\star})]\nonumber \\
&=& \label{eq:efficiency}
\mathbb{E}[\Phi_B^\star]
- \sum_{k=1}^K
\mathbb{E}[\Phi(S_k^\star)] \\ \label{eq:objective} && + \sum_{k=1}^K
\mathbb{E}[\Phi(S_k^\star)] -
\mathbb{E}[\Phi({\bigcup_{k=1}^K S_k^\star})]
\end{eqnarray}

The expression in (\ref{eq:efficiency}) captures the difference between the objectives optimized by the two methods. In the very low-budget regime, we assume that each selection - whether local or global - is efficient, selecting at most one point per cell. Thus, uncoordinated but efficient selection yields
$\mathbb{E}[\Phi_B^\star]
\approx
\sum_{k=1}^K \mathbb{E}[\Phi(S_k^\star)]$
(see Appendix~\ref{app:coverage-gap-proofs}).

The key difference therefore lies in (\ref{eq:objective}), which measures how much of the utility independently gained by individual clients is preserved in the combined selected set rather than lost to redundant selections. We thus focus on the \emph{expected redundancy gap} $\bar{\Delta}_K^{\mathrm{red}}$, defined as
\begin{equation}
\label{eq:effective-gap}
\bar{\Delta}_K^{\mathrm{red}}
=
\sum_{k=1}^K
\mathbb{E}[\Phi(S_k^\star)]
-
\mathbb{E}\!\left[\Phi\!\left(\bigcup_{k=1}^K S_k^\star\right)\right].
\end{equation}
Under the aforementioned efficient-selection assumption,
\begin{equation}
\label{eq:coord-red-approx}
\bar{\Delta}_K
\approx
\bar{\Delta}_K^{\mathrm{red}}.
\end{equation}
Thus, in what follows, we focus on the redundancy gap.

\begin{proposition}[Heterogeneity and the Redundancy Gap]
\label{prop:heterogeneity-gap}
Assume that, for each cell $C_m$, the events
$\{S_k^\star\cap C_m\neq\emptyset\}$ are independent across clients.
Then the following hold:
\begin{enumerate}
    \item If the accessible cell sets are pairwise disjoint,
    \[
    A_i\cap A_j=\emptyset
    \qquad
    \forall i\neq j,
    \]
    then
    \[
    \bar{\Delta}_K^{\mathrm{red}}=0.
    \]

    \item Define the contribution of cell $C_m$ to the expected redundancy gap as
    \[
    \Delta_m
    =
    \sum_{k=1}^K a_{k,m}
    -
    \left(
    1-\prod_{k=1}^K(1-a_{k,m})
    \right).
    \]
    Then
    \[
    \frac{\partial\Delta_m}{\partial a_{j,m}}
    =
    1-\prod_{k\neq j}(1-a_{k,m})
    \geq 0,
    \]
    with strict inequality whenever another client selects from $C_m$ with positive probability.

    \item Let
    \[
    c_m^{(K)}
    =
    1-\prod_{k=1}^K(1-a_{k,m})
    \]
    be the probability that at least one of the first $K$ clients selects from $C_m$. Then
    \[
    \bar{\Delta}_{K+1}^{\mathrm{red}}-\bar{\Delta}_K^{\mathrm{red}}
    =
    \sum_{m=1}^M
    a_{K+1,m}c_m^{(K)}.
    \]
\end{enumerate}
\end{proposition}
\noindent
Proof is provide in Appendix~\ref{app:propo1-proof}.

The last result shows that the marginal redundancy introduced by an additional client is determined by its overlap with the existing selection profile. Heterogeneity that shifts selection mass toward cells with small $c_m^{(K)}$ therefore causes a smaller increase in the \emph{redundancy gap}.

\begin{corollary}
\label{cor:iid-gap} 
If the clients' distributions are IID, the \emph{redundancy gap} increases (or remains unchanged) with the number of clients.
\end{corollary}
\begin{proof}
If all clients have identical cell-level selection probabilities
$a_{1,m}
=
\cdots
=
a_{K,m}
=
a_m$,
then
\begin{equation}
\label{eq:iid-gap}
\bar{\Delta}_K^{\mathrm{red}}
=
\sum_{m=1}^M
\left[
K a_m
-
\left(
1-(1-a_m)^K
\right)
\right],
\end{equation}
and
\begin{equation}
\label{eq:iid-gap-increment}
\bar{\Delta}_{K+1}^{\mathrm{red}}
-
\bar{\Delta}_K^{\mathrm{red}}
=
\sum_{m=1}^M
a_m
\left[
1-(1-a_m)^K
\right]
\geq 0.
\end{equation}  
\end{proof}

\begin{corollary}Adding a client whose accessible cells do not overlap with those of the existing clients introduces no additional redundancy.
\label{cor:disjoint-client}
\end{corollary}
\begin{proof}
\phantom\newline\newline
$
A_{K+1}\cap\bigcup_{k=1}^K A_k=\emptyset
~~ \implies ~~
\bar{\Delta}_{K+1}^{\mathrm{red}}
-
\bar{\Delta}_K^{\mathrm{red}}
=
0.
$
\end{proof}

Together, these results explain why homogeneous partitions may require stronger coordination than heterogeneous partitions: IID clients repeatedly spend their budgets on similar regions, whereas heterogeneity naturally reduces duplicated coverage.

\begin{observation}
\label{obs:neg_coor}
The theoretical coordination gap in Eq.~\eqref{eq:gap} is nonnegative because it is defined relative to the centralized optimum. The empirical gap between realizable centralized and local selection methods may nevertheless be negative. This can occur when a non-IID partition reveals semantic structure that is not fully represented by the embedding geometry, allowing independent local selection to outperform the centralized selection method (see Figure~\ref{fig:same_embedding_noniid_gap}).
\end{observation}

\section{Method}
\label{sec:method}

We propose a three-phase FAL framework for cross-silo settings\footnote{While 'cross-silo' more precisely describes our setting, we henceforth use 'client' to align with prevailing FL nomenclature.}: (i) federated representation learning, (ii) privacy-conscious centralized active selection under client-level budgets, and (iii) federated downstream training (Figure~\ref{fig:method_visualization}). Raw data remain local throughout; only perturbed representations or their obfuscated summaries are communicated for query coordination.

\subsection{Three-Phase Pipeline}

\paragraph{Phase I: Federated Representation Learning}

The first phase (Figure~\ref{fig:method_visualization}, highlighted in teal) constructs a shared representation space across clients, enabling coordination during the subsequent active selection stage. Accordingly, we train a shared feature extractor using standard Federated Learning (FL), adopting FedAvg \citep{mcmahan2017communication} to train a SimCLR encoder in a federated manner (see Section~\ref{sec:related}). At each communication round, the server broadcasts the current encoder, clients perform local contrastive learning on their unlabeled data, and the server aggregates the resulting updates. This allows clients to collaboratively learn a common representation without exchanging raw data. Once training is complete, the global encoder is frozen and distributed to all clients, which use it to embed their local unlabeled data.

\paragraph{Phase II: Obfuscated Centralized Active Selection}
The second phase (Figure~\ref{fig:method_visualization}, highlighted purple) 
performs centralized active query selection over the clients' unlabeled data. 
To this end we adopt three methods that have demonstrated superior performance in the low-budget regime - \emph{ProbCover} \citep{yehuda2022active}, \emph{TypiClust} \citep{hacohen2022active} and \emph{MaxHerding} \citep{bae2024maxherding}, adapted to the current setting by adding local client budget constraints.

\subsubsection*{Client-Constrained Selection}

We apply centralized AL selection globally in the shared embedding space while enforcing client-level annotation budgets. For ProbCover and MaxHerding, the server greedily selects the feasible sample with the largest marginal coverage gain, excluding samples from clients whose budgets have been exhausted. For TypiClust, clustering is performed globally, and cluster representatives are allocated subject to the same client budgets. The formal objective, greedy procedure, and approximation guarantee are provided in Appendix~\ref{app:client_constrained_selection}.

\subsubsection*{Obfuscation}
\label{sec:obfuscate}

Directly transmitting embeddings may still introduce privacy risks, including reconstruction or inversion attacks. We therefore study two approaches for obfuscating the per-client embeddings before they are used for centralized selection. Both approaches aim to preserve enough geometric information for effective selection, while reducing the amount of information exposed to the server.

\paragraph{Approach 1: using controlled perturbation.}

\paragraph{Approach 2: using data aggregates.}
We follow the principle of geometric summarization
\citep[e.g.,][]{ye2021fedproto}, whereby clients communicate only
aggregate geometric information rather than individual embeddings.
To enable coordinated active selection under this constraint, we
introduce federated adaptations of ProbCover and TypiClust in Section~\ref{sec:adapted-method}.

\paragraph{Phase III: Federated Training of a Global Classifier}

In Phase III (Figure~\ref{fig:method_visualization}, highlighted blue) selected samples remain local and are used for federated downstream training. We evaluate a neural network trained on raw images, and as an effective alternative in low budgets, a shallow classifier trained over the shared representation learned in Phase I.

\subsection{Aggregate-Based Federated Active Selection}
\label{sec:adapted-method}

Sharing individual embeddings enables global coordination, but may reveal
information about individual samples. We therefore introduce two federated
adaptations of geometry-based active learning,
\emph{FederatedProbCover} and \emph{FederatedTypiClust}, that coordinate
selection across clients while communicating only geometric aggregates and
distances. Unlike independent client selection, these methods recover
the global geometric information needed for coordinated querying without
exposing individual embeddings.

\paragraph{FederatedProbCover (Alg.~\ref{alg:fedprobcover}).}
The key challenge in federating ProbCover is that its greedy criterion
depends on the \emph{global} number of samples newly covered by each query.
We replace direct access to the global embedding pool with a distributed
proposal-and-count procedure. At each round, each client performs a local
ProbCover step and sends only the centroid of the region induced by its best
candidate. The server broadcasts these proposals, and clients return the
number of currently uncovered local samples covered by each centroid.
Summing these counts yields the global coverage gain of every proposal
without revealing the covered samples themselves.

\begin{figure*}[t!]
    \centering
    \begin{subfigure}[t]{0.48\linewidth}
        \centering
        \includegraphics[width=\linewidth]{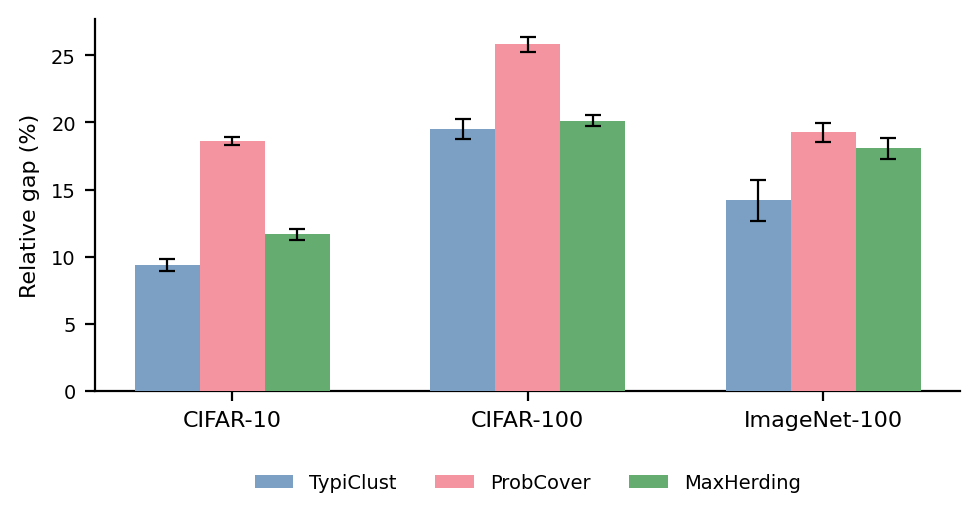}
        \caption{IID client partitions.}
        \label{fig:iid_coordination_gap}
    \end{subfigure}
    \hfill
    \begin{subfigure}[t]{0.48\linewidth}
        \centering
        \includegraphics[width=\linewidth]{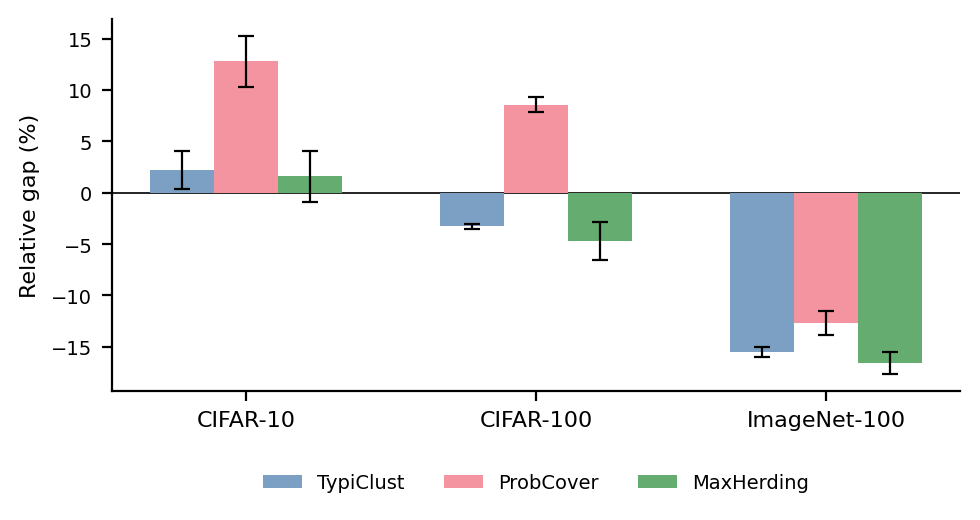}
        \caption{Non-IID client partitions with $\alpha=0.1$.}
        \label{fig:non_iid_coordination_gap}
    \end{subfigure}

    \caption{
    Empirical coordination gap across datasets and selection methods.
    Error bars denote standard error over 3 random client splits.
    }
    \label{fig:coordination_gap}
\end{figure*}

The server selects the
proposal with maximal global gain and queries the closest feasible sample
to that centroid across clients, subject to client annotation budgets.
All clients then update their local coverage state using the selected
centroid. Thus, subsequent rounds account for coverage already obtained
on \emph{other clients}, explicitly reducing the cross-client redundancy
of independent selection.

\begin{algorithm}[t]
\caption{\textsc{FederatedProbCover}}
\label{alg:fedprobcover}
\small
\begin{algorithmic}[1]
\While{annotation budget remains}
    \State Each client $k$ proposes centroid $c_k$ of its locally best
    uncovered region.
    \State Server broadcasts $\{c_k\}$ to all clients.
    \State Each client $j$ reports local coverage gain $g_{jk}$ for
    every $c_k$.
    \State Server selects
    $c^*=\arg\max_{c_k}\sum_j g_{jk}$.
    \State Each client reports its closest feasible sample to $c^*$
    and its distance.
    \State Server queries the globally closest candidate subject to
    client budgets.
    \State All clients mark samples covered by $c^*$ as covered.
\EndWhile
\end{algorithmic}
\end{algorithm}

\paragraph{FederatedTypiClust.}
TypiClust admits a particularly natural federated adaptation because its
selection criterion is already defined through cluster-level geometry.
We first construct global clusters using federated $k$-means: clients
communicate only per-cluster sums and counts, from which the server updates
the global centroids. Given these centroids, each client reports, for every
cluster, the distance to its closest eligible local sample. The server
orders clusters according to the TypiClust criterion - prioritizing clusters
containing fewer labeled samples and, as a tie-breaker, larger clusters - and
queries the globally closest feasible sample to each selected centroid,
while enforcing the per-client annotation budgets.

The two adaptations recover complementary forms of global geometric
information from aggregates: FederatedProbCover estimates
\emph{global marginal coverage} through distributed counting, whereas
FederatedTypiClust recovers \emph{global cluster structure} through
federated sufficient statistics. Neither requires transmitting individual
embeddings while retaining cross-client coordination.

\section{Empirical Results}

\begin{table*}[t!]
\centering
\small
\begin{tabular}{lccccccccc}
\toprule
& \multicolumn{3}{c}{CIFAR-10}
& \multicolumn{3}{c}{CIFAR-100}
& \multicolumn{3}{c}{ImageNet-100} \\
\cmidrule(lr){2-4}
\cmidrule(lr){5-7}
\cmidrule(lr){8-10}
Method & 10 & 50 & 100 & 100 & 500 & 1000 & 100 & 500 & 1000 \\
\midrule
\multicolumn{10}{l}{\textit{IID clients}} \\

Our method (ProbCover)
& \textbf{49.5$_{\pm3.1}$}
& \textbf{78.7$_{\pm0.5}$}
& \textbf{84.4$_{\pm0.1}$}
& \textbf{41.3$_{\pm0.1}$}
& \textbf{48.6$_{\pm0.6}$}
& \textbf{51.2$_{\pm0.6}$}
& \textbf{40.2$_{\pm0.6}$}
& \textbf{50.6$_{\pm0.4}$}
& \textbf{54.9$_{\pm0.1}$} \\

Ono et al. (ProbCover)
& 32.3$_{\pm3.0}$
& 64.8$_{\pm1.1}$
& 72.9$_{\pm1.1}$
& 20.8$_{\pm0.5}$
& 36.3$_{\pm0.3}$
& 41.3$_{\pm0.1}$
& 24.8$_{\pm0.9}$
& 41.4$_{\pm0.4}$
& 46.6$_{\pm0.2}$ \\

\cdashline{1-10}

Our method (TypiClust)
& \textbf{32.6$_{\pm5.5}$}
& \textbf{78.2$_{\pm0.7}$}
& \textbf{82.2$_{\pm0.7}$}
& \textbf{37.7$_{\pm0.5}$}
& \textbf{46.3$_{\pm0.1}$}
& \textbf{50.1$_{\pm0.1}$}
& \textbf{38.9$_{\pm0.7}$}
& \textbf{49.5$_{\pm0.5}$}
& \textbf{53.8$_{\pm0.4}$} \\

Ono et al. (TypiClust)
& 29.2$_{\pm5.2}$
& 69.4$_{\pm0.1}$
& 78.9$_{\pm0.5}$
& 25.7$_{\pm0.7}$
& 37.4$_{\pm0.4}$
& 41.1$_{\pm0.6}$
& 28.9$_{\pm0.8}$
& 42.6$_{\pm0.2}$
& 46.2$_{\pm1.0}$ \\

\cdashline{1-10}

Our method (MaxHerding)
& \textbf{61.7$_{\pm2.6}$}
& \textbf{83.1$_{\pm0.5}$}
& \textbf{84.2$_{\pm0.2}$}
& \textbf{36.2$_{\pm0.4}$}
& \textbf{46.9$_{\pm0.2}$}
& \textbf{50.2$_{\pm0.1}$}
& \textbf{36.1$_{\pm0.8}$}
& \textbf{49.1$_{\pm0.2}$}
& \textbf{54.2$_{\pm0.5}$} \\

Ono et al. (MaxHerding)
& 46.1$_{\pm1.6}$
& 71.2$_{\pm0.9}$
& 79.2$_{\pm0.4}$
& 22.5$_{\pm0.4}$
& 37.9$_{\pm0.3}$
& 41.6$_{\pm0.1}$
& 18.3$_{\pm0.1}$
& 40.3$_{\pm0.3}$
& 46.6$_{\pm0.5}$ \\

\midrule
\multicolumn{10}{l}{\textit{Non-IID clients ($\alpha=0.1$)}} \\

Our method (ProbCover)
& \textbf{35.3$_{\pm4.3}$}
& \textbf{66.5$_{\pm2.4}$}
& \textbf{77.8$_{\pm1.3}$}
& \textbf{31.5$_{\pm0.8}$}
& \textbf{44.7$_{\pm0.8}$}
& \textbf{48.3$_{\pm0.4}$}
& 32.2$_{\pm0.2}$
& 44.8$_{\pm0.1}$
& 49.7$_{\pm0.2}$ \\

Ono et al. (ProbCover)
& 26.0$_{\pm0.6}$
& 56.8$_{\pm1.5}$
& 69.8$_{\pm0.4}$
& 24.5$_{\pm0.7}$
& 40.8$_{\pm0.3}$
& 46.4$_{\pm0.2}$
& \textbf{32.6$_{\pm0.4}$}
& \textbf{50.5$_{\pm0.8}$}
& \textbf{56.5$_{\pm0.8}$} \\

\cdashline{1-10}

Our method (TypiClust)
& \textbf{31.9$_{\pm2.4}$}
& 67.2$_{\pm2.0}$
& \textbf{76.7$_{\pm1.8}$}
& \textbf{28.7$_{\pm0.8}$}
& 41.4$_{\pm0.5}$
& 46.0$_{\pm0.4}$
& 31.3$_{\pm0.9}$
& 43.2$_{\pm0.4}$
& 49.3$_{\pm0.3}$ \\

Ono et al. (TypiClust)
& 26.4$_{\pm1.4}$
& \textbf{67.4$_{\pm1.3}$}
& 76.3$_{\pm0.3}$
& 27.8$_{\pm1.2}$
& \textbf{43.4$_{\pm0.3}$}
& \textbf{47.1$_{\pm0.1}$}
& \textbf{37.7$_{\pm0.5}$}
& \textbf{50.8$_{\pm0.7}$}
& \textbf{54.8$_{\pm0.3}$} \\

\cdashline{1-10}

Our method (MaxHerding)
& \textbf{54.4$_{\pm2.3}$}
& 71.8$_{\pm2.4}$
& 78.5$_{\pm1.5}$
& 30.2$_{\pm1.5}$
& 43.1$_{\pm0.9}$
& 47.3$_{\pm0.2}$
& 30.9$_{\pm0.3}$
& 44.6$_{\pm0.2}$
& 49.4$_{\pm0.3}$ \\

Ono et al. (MaxHerding)
& 43.2$_{\pm4.9}$
& \textbf{73.5$_{\pm0.9}$}
& \textbf{78.8$_{\pm0.5}$}
& \textbf{31.5$_{\pm0.5}$}
& \textbf{46.1$_{\pm0.4}$}
& \textbf{48.6$_{\pm0.2}$}
& \textbf{37.7$_{\pm1.8}$}
& \textbf{51.6$_{\pm0.9}$}
& \textbf{56.1$_{\pm0.3}$} \\

\bottomrule
\end{tabular}

\caption{
Full-pipeline comparison between our method and the baseline adapted from \citet{ono2025exploring}. Our method performs globally coordinated selection in a shared federated embedding space and trains a shared classifier using FedAvg. The baseline performs selection and classifier training independently at each client and uses the prediction of the most confident client classifier at inference. Each pair uses the same active selection method. Entries report mean test accuracy (\%) $\pm$ standard error over three seeds at total annotation budgets corresponding to 1, 5, and 10 labeled samples per class. Bold indicates the best outcome within each pair.
}
\label{tab:conf_iid_noniid}
\end{table*}

\subsection{Evaluation Score}
\label{sec:results_evaluation}

The coordination gap $\Delta_K$ in (\ref{eq:gap}) measures the accuracy difference at a given annotation budget. We summarize this gap across budgets using its normalized AUC counterpart, termed \emph{Empirical Coordination Gap} and defined as follows:

\begin{definition}[Empirical Coordination Gap]
Let $\mathrm{AUC}_{\mathrm{cent}}$ and $\mathrm{AUC}_{\mathrm{pc}}$ denote the empirical areas under the centralized and per-client accuracy--budget curves, computed by trapezoidal integration over the evaluated budgets. We define the empirical coordination gap as
\[
\widehat{\Delta}_{\mathrm{AUC}}
=
100\cdot
\frac{\mathrm{AUC}_{\mathrm{cent}}-\mathrm{AUC}_{\mathrm{pc}}}
{\mathrm{AUC}_{\mathrm{cent}}}.
\]
\end{definition}

\subsection{Results: Full Pipeline}
\label{sec:non_iid_full_pipeline}
\label{sec:full-pipeline}

We first evaluate the complete pipeline without embedding obfuscation, comparing our coordinated approach with the strengthened per-client baseline described in Appendix~\ref{sec:experimental_methodology}. Both use the same shallow probabilistic classifier architecture; the effects of embedding obfuscation are evaluated separately in Section~\ref{sec:privacy_exp}.

\paragraph{IID Client Distributions}

The results for IID clients are shown in Figure~\ref{fig:iid_coordination_gap} and Table~\ref{tab:conf_iid_noniid}. In Figure~\ref{fig:high_budget_fal_comparison} we further compare our method against high budget FAL-specific approaches proposed in \citep{cao2023knowledge, kim2023re}. Across datasets and selection methods, our centralized pipeline consistently improves over the per-client baseline, demonstrating a positive gap under IID partitions; this agrees with the redundancy-based prediction of Section~\ref{sec:hetero}.

\begin{figure}[t!]
    \centering
    \includegraphics[width=\linewidth]{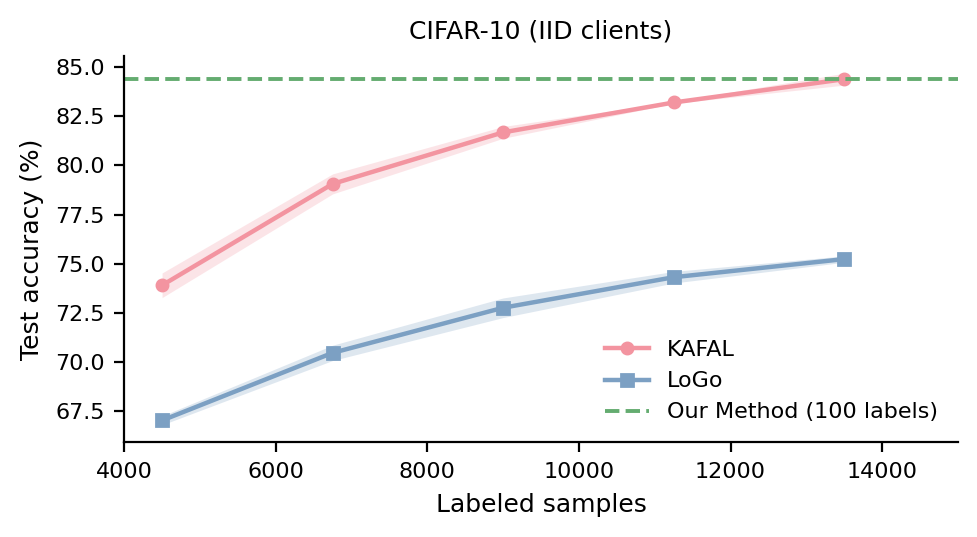}
    \caption{
       Comparison against two representative high-budget FAL methods \citep{cao2023knowledge, kim2023re}. The dashed line marks the accuracy achieved by our method using ProbCover selection with a substantially smaller annotation budget, highlighting the effectiveness of coordinated low-budget selection even relative to methods evaluated with larger budgets.}
    \label{fig:high_budget_fal_comparison}
\end{figure}

Implementation details and methodology are provided in Appendix~\ref{sec:experimental_methodology}.

\begin{figure}[t!]        
\centering
\includegraphics[width=\linewidth]{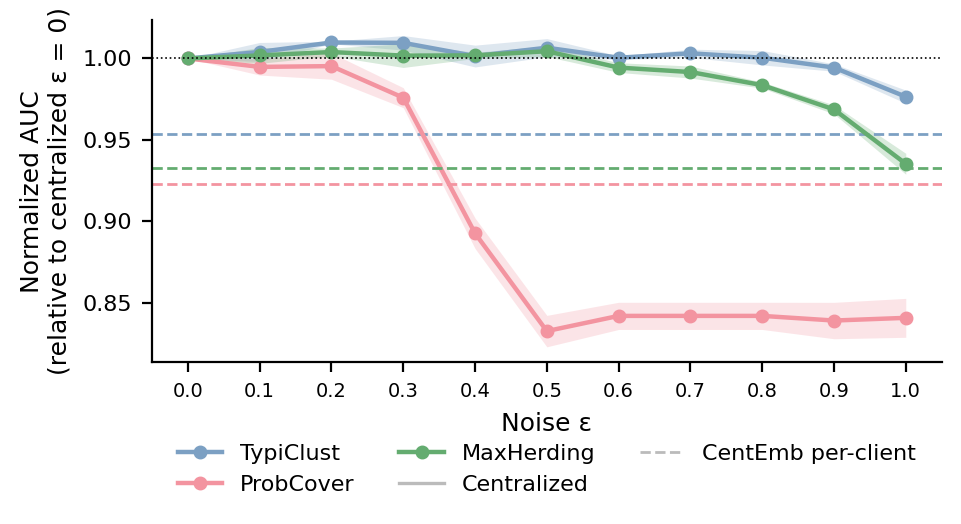}
\caption{The normalized AUC as a function of the embedding noise level. Dashed lines show the corresponding aligned embedding per-client baselines for each method.}
\label{fig:noise_auc_vs_eps}
\vspace{-.5cm}
\end{figure}

\paragraph{Non-IID Client Distributions}

The results for non-IID clients are shown in  Figure~\ref{fig:non_iid_coordination_gap} and Table~\ref{tab:conf_iid_noniid}. Across datasets and selection methods, the coordination gap now decreases substantially relative to the IID setting and sometimes becomes negative.
This reduction is consistent with the predicted heterogeneity reversal analyzed in Section~\ref{sec:hetero}. The negative gaps in some configurations are also consistent with Observation~\ref{obs:neg_coor}, as  maximum-confidence aggregation may further benefit from client specialization.

\subsection{Privacy-Preserving Data Obfuscation}
\label{sec:privacy_exp}

In this section, we evaluate the two obfuscation mechanisms introduced in Section~\ref{sec:obfuscate}: controlled embedding perturbation and centroid-based communication.

\begin{figure*}[t]
    \centering
    \begin{subfigure}[t]{0.48\linewidth}
        \centering
        \includegraphics[width=\linewidth]{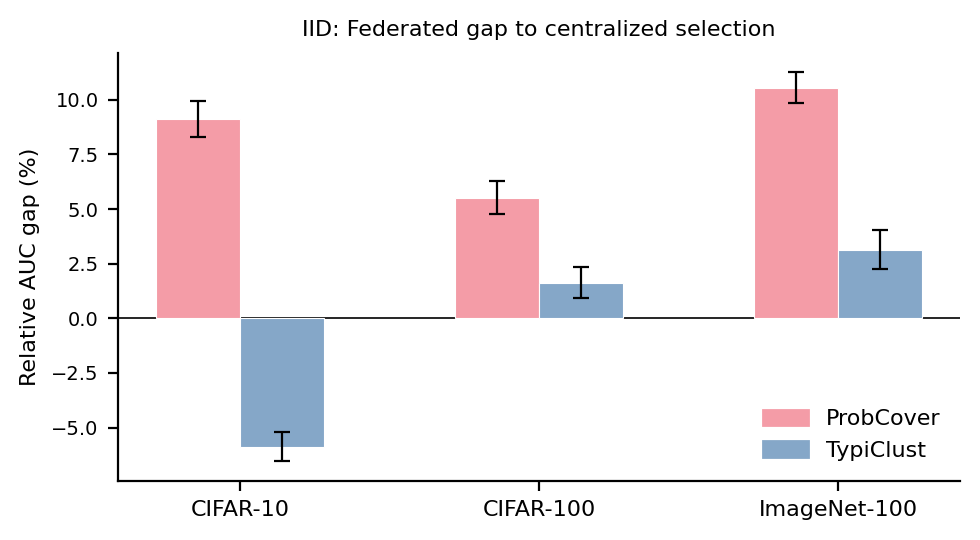}
        \caption{
        The \emph{empirical coordination gap} between centralized selection and its federated counterpart. Lower indicates closer agreement with centralized selection.
        }
        \label{fig:fed_classic_iid_fed_vs_cent}
    \end{subfigure}
    \hfill
    \begin{subfigure}[t]{0.48\linewidth}
        \centering
        \includegraphics[width=\linewidth]{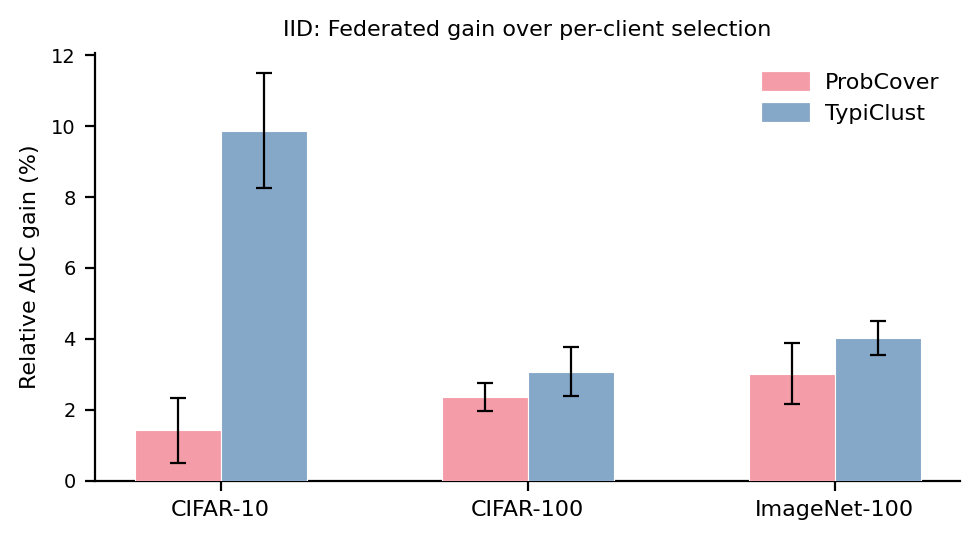}
        \caption{
        The \emph{empirical coordination gap} over independent per-client selection in the shared embedding space. Higher indicates stronger benefit from coordination.
        }
        \label{fig:fed_classic_iid_fed_vs_cemb}
    \end{subfigure}

    \caption{
    Centroid-based federated variants of ProbCover and TypiClust under IID client partitions.
    \textbf{Left:} remaining gap  to fully centralized selection.
    \textbf{Right:} improvement over independent per-client selection in the shared embedding space.
    Both quantities are computed analogously to the empirical coordination gap; error bars denote standard error over three seeds.
    }
    \label{fig:fed_classic_iid}
\end{figure*}

\begin{table*}[t]
\centering
\small
\begin{tabular}{lccccccccc}
\toprule
& \multicolumn{3}{c}{CIFAR-10} & \multicolumn{3}{c}{CIFAR-100} & \multicolumn{3}{c}{ImageNet-100} \\
\cmidrule(lr){2-4} \cmidrule(lr){5-7} \cmidrule(lr){8-10}
Method & 10 & 50 & 100 & 100 & 500 & 1000 & 100 & 500 & 1000 \\
\midrule
\multicolumn{10}{l}{\textit{ProbCover}} \\
Centralized & \textbf{49.5$_{\pm3.1}$} & \textbf{78.7$_{\pm0.5}$} & \textbf{84.4$_{\pm0.1}$} & 41.3$_{\pm0.1}$ & \textbf{48.6$_{\pm0.6}$} & \textbf{51.2$_{\pm0.6}$} & 40.2$_{\pm0.6}$ & \textbf{50.6$_{\pm0.4}$} & \textbf{54.9$_{\pm0.1}$} \\
Federated & 48.2$_{\pm3.3}$ & 73.9$_{\pm2.0}$ & 79.1$_{\pm0.3}$ & \textbf{42.0$_{\pm0.1}$} & 45.3$_{\pm0.2}$ & 48.8$_{\pm0.3}$ & \textbf{40.5$_{\pm0.5}$} & 44.4$_{\pm0.3}$ & 48.8$_{\pm0.2}$ \\
CentEmb per-client & 38.1$_{\pm3.0}$ & 73.8$_{\pm1.1}$ & 79.1$_{\pm0.1}$ & 27.3$_{\pm0.4}$ & 45.2$_{\pm0.2}$ & 49.8$_{\pm0.4}$ & 23.4$_{\pm1.0}$ & 44.4$_{\pm0.1}$ & 50.2$_{\pm0.5}$ \\
\cdashline{1-10}
\multicolumn{10}{l}{\textit{TypiClust}} \\
Centralized & 32.6$_{\pm5.5}$ & 78.2$_{\pm0.7}$ & \textbf{82.2$_{\pm0.7}$} & 37.7$_{\pm0.5}$ & \textbf{46.3$_{\pm0.1}$} & \textbf{50.1$_{\pm0.1}$} & \textbf{38.9$_{\pm0.7}$} & \textbf{49.5$_{\pm0.5}$} & \textbf{53.8$_{\pm0.4}$} \\
Federated & \textbf{62.2$_{\pm1.4}$} & \textbf{79.9$_{\pm1.2}$} & 81.8$_{\pm0.3}$ & \textbf{37.8$_{\pm0.2}$} & 45.0$_{\pm0.5}$ & 49.3$_{\pm0.3}$ & 36.3$_{\pm0.7}$ & 47.8$_{\pm0.8}$ & 53.0$_{\pm0.2}$ \\
CentEmb per-client & 31.8$_{\pm5.2}$ & 73.2$_{\pm0.8}$ & \textbf{82.2$_{\pm0.3}$} & 30.9$_{\pm0.4}$ & 44.6$_{\pm0.2}$ & 47.9$_{\pm0.2}$ & 29.6$_{\pm0.8}$ & 46.0$_{\pm0.3}$ & 51.1$_{\pm0.5}$ \\
\bottomrule
\end{tabular}
\caption{Test accuracy (\%) at 1, 5, and 10 labeled samples per class, comparing fully centralized selection, the federated centroid-based variant (FedProbCover/FedTypiClust), and the noiseless aligned-embedding per-client baseline (CentEmb per-client), under IID client partitions. Mean $\pm$ standard error over 3 seeds. Bold marks the best of the three rows per column, within each method block.}
\label{tab:federated_algorithms}
\end{table*}

\paragraph{Controlled embedding perturbation.}
We compare noisy centralized selection with independent per-client selection in the aligned embedding space. The noise level $\epsilon$ denotes the target expected $\ell_2$ displacement between each original and perturbed unit-normalized embedding. TypiClust and MaxHerding incur little or no accuracy loss up to $\epsilon=0.6$, whereas ProbCover is more sensitive to perturbation, see Figure~\ref{fig:noise_auc_vs_eps} and Appendix~\ref{app:noise} for details.

\paragraph{Federated Versions of the ``Classic'' AL Selection Algorithms.} We also evaluate the two aggregate-based FAL methods described in Section~\ref{sec:adapted-method}  under IID client partitions. Figure~\ref{fig:fed_classic_iid} and Table~\ref{tab:federated_algorithms} report their remaining gap to centralized selection and their improvement over independent per-client selection in the shared embedding space. Both outperform independent per-client selection, indicating that aggregate-based cross-client coordination reduces query redundancy. FederatedTypiClust closely matches centralized selection, whereas FederatedProbCover retains a larger gap.

\begin{figure*}[t!]
    \centering

    \begin{subfigure}[t]{0.48\linewidth}
        \centering
        \includegraphics[width=\linewidth]{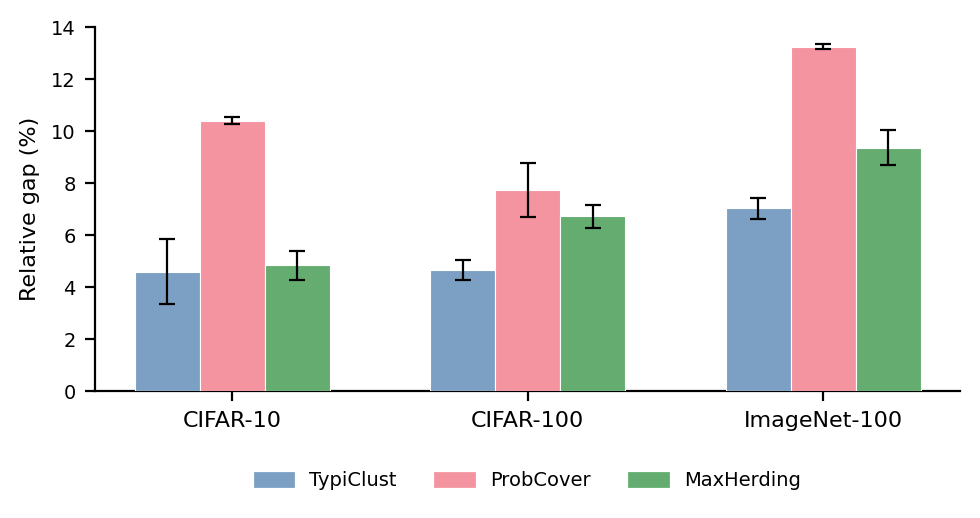}
        \caption{IID client partitions.}
        \label{fig:same_embedding_iid_gap}
    \end{subfigure}
    \hfill
    \begin{subfigure}[t]{0.48\linewidth}
        \centering
        \includegraphics[width=\linewidth]{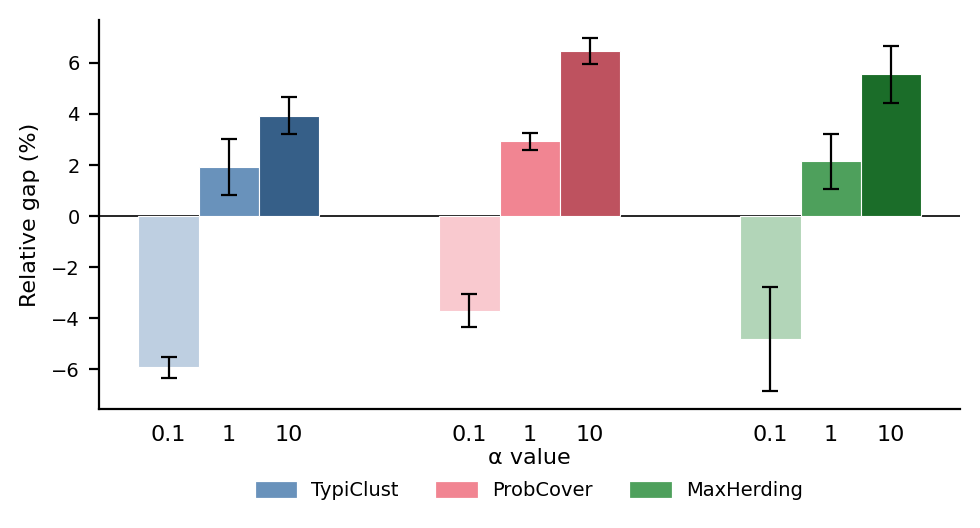}
        \caption{Non-IID client partitions.}
        \label{fig:same_embedding_noniid_gap}
    \end{subfigure}

    \caption{
    Same-embedding selection-stage ablation.
    In both settings, centralized and per-client selection operate on the same federated embedding, isolating the effect of global query coordination from representation quality.
    \textbf{Left:} Empirical coordination gap under IID client partitions across datasets and selection methods.
    \textbf{Right:} Empirical coordination gap under non-IID client partitions as the Dirichlet parameter $\alpha$ varies.
    Positive values indicate that centralized selection improves over independent per-client selection, while negative values indicate that per-client selection performs better.
    }
    \label{fig:same_embedding_auc_gap}
\end{figure*}

\subsection{Ablation Study}
\label{sec:same_embedding_ablation}

To isolate query coordination from representation quality, we compare centralized and per-client selection within the same federated embedding. Figure~\ref{fig:same_embedding_auc_gap} shows a positive coordination benefit under IID partitions that decreases and can become negative with increasing heterogeneity, mirroring the full-pipeline trend. This supports the interpretation that the reversal arises from the selection stage itself. Absolute accuracies are reported in Appendix~\ref{sec:ablation}.

\section{Discussion}
\label{sec:discussion}

Our results highlight the importance of globally coordinated query selection in low-budget FAL. By using a shared federated representation, our framework enables selection across clients while keeping raw data private and downstream training federated. This allows standard low-budget AL methods to reduce cross-client redundancy without violating privacy constraints.
Our results also show that coordination utility depends strongly on client heterogeneity: its benefit is largest under homogeneous partitions and declines as heterogeneity itself provides cross-client diversity.

These results suggest that low-budget FAL is fundamentally a joint selection--representation problem. Coordination is most beneficial when clients are similar enough to produce redundant local queries and can be compared meaningfully in a shared representation space. More broadly, our framework provides a modular foundation for future FAL methods that jointly adapt selection, representation learning, and privacy-preserving coordination.

\subsection*{\textbf{Acknowledgments}}
This work was supported by a grant from the Gatsby Charitable Foundation and AFOSR award FA8655-24-1-7006.

\bibliographystyle{plainnat}
\bibliography{refs}

@inproceedings{deng2009imagenet,
  title={ImageNet: A Large-Scale Hierarchical Image Database},
  author={Jia Deng and Wei Dong and Richard Socher and Li-Jia Li and Kai Li and Li Fei-Fei},
  booktitle={IEEE Conference on Computer Vision and Pattern Recognition},
  year={2009}
}

@article{kairouz2021advances,
  title={Advances and Open Problems in Federated Learning},
  author={Peter Kairouz and H. Brendan McMahan and Brendan Avent and Aurélien Bellet and Mehdi Bennis and Arjun Nitin Bhagoji and Keith Bonawitz and Zachary Charles and Graham Cormode and Rachel Cummings},
  journal={Foundations and Trends in Machine Learning},
  year={2021}
}

@inproceedings{lewis1994sequential,
  title={A Sequential Algorithm for Training Text Classifiers},
  author={David D. Lewis and William A. Gale},
  booktitle={SIGIR},
  year={1994}
}

@inproceedings{hacohen2022active,
  title={Active Learning on a Budget: Opposite Strategies Suit High and Low Budgets},
  author={Guy Hacohen and Avihu Dekel and Daphna Weinshall},
  booktitle={ICML},
  year={2022}
}

@inproceedings{bae2024maxherding,
  title={Generalized coverage for more robust low-budget active learning},
  author={Bae, Wonho and Noh, Junhyug and Sutherland, Danica J},
  booktitle={European conference on computer vision},
  pages={318--334},
  year={2024},
  organization={Springer}
}

@article{goetz2019active,
  title={Active federated learning},
  author={Goetz, Jack and Malik, Kshitiz and Bui, Duc and Moon, Seungwhan and Liu, Honglei and Kumar, Anuj},
  journal={arXiv preprint arXiv:1909.12641},
  year={2019}
}

@article{ahn2024federated,
  title={Federated active learning (f-al): an efficient annotation strategy for federated learning},
  author={Ahn, Jin-Hyun and Ma, Yeeun and Park, Seoyun and You, Cheolwoo},
  journal={IEEE Access},
  volume={12},
  pages={39261--39269},
  year={2024},
  publisher={IEEE}
}

@inproceedings{kim2022lg,
  title={Lg-fal: Federated active learning strategy using local and global models},
  author={Kim, S and Bae, S and Yun, Se-Young and Song, Hwanjun},
  booktitle={Proceedings of the ICML Workshop on Adaptive Experimental Design and Active Learning in the Real World, Baltimore, MD, USA},
  pages={127--139},
  year={2022}
}

@inproceedings{cao2023knowledge,
  title={Knowledge-aware federated active learning with non-iid data},
  author={Cao, Yu-Tong and Shi, Ye and Yu, Baosheng and Wang, Jingya and Tao, Dacheng},
  booktitle={Proceedings of the IEEE/CVF International Conference on Computer Vision},
  pages={22279--22289},
  year={2023}
}

@inproceedings{zhang2023affectfal,
  title={AffectFAL: Federated active affective computing with non-IID data},
  author={Zhang, Zixin and Qi, Fan and Li, Shuai and Xu, Changsheng},
  booktitle={Proceedings of the 31st ACM International Conference on Multimedia},
  pages={871--882},
  year={2023}
}

@inproceedings{kim2023re,
  title={Re-thinking federated active learning based on inter-class diversity},
  author={Kim, SangMook and Bae, Sangmin and Song, Hwanjun and Yun, Se-Young},
  booktitle={Proceedings of the IEEE/CVF Conference on Computer Vision and Pattern Recognition},
  pages={3944--3953},
  year={2023}
}

@inproceedings{jing2024fedsc,
  title={FedSC: Provable Federated Self-supervised Learning with Spectral Contrastive Objective over Non-iid Data},
  author={Jing, Shusen and Yu, Anlan and Zhang, Shuai and Zhang, Songyang},
  booktitle={International Conference on Machine Learning},
  pages={22304--22325},
  year={2024},
  organization={PMLR}
}

@inproceedings{ono2025exploring,
  title={Exploring the Possibility of TypiClust for Low-Budget Federated Active Learning},
  author={Ono, Yuta and Nakamura, Hiroshi and Takase, Hideki},
  booktitle={2025 IEEE 49th Annual Computers, Software, and Applications Conference (COMPSAC)},
  pages={648--653},
  year={2025},
  organization={IEEE}
}

@article{zhuang2021divergence,
  title={Divergence-aware federated self-supervised learning},
  author={Zhuang, Weiming and Wen, Yonggang and Zhang, Shuai},
  journal={arXiv preprint arXiv:2204.04385},
  year={2022}
}

@inproceedings{ye2021fedproto,
  title={Fedproto: Federated prototype learning across heterogeneous clients},
  author={Tan, Yue and Long, Guodong and Liu, Lu and Zhou, Tianyi and Lu, Qinghua and Jiang, Jing and Zhang, Chengqi},
  booktitle={Proceedings of the AAAI conference on artificial intelligence},
  volume={36:8},
  pages={8432--8440},
  year={2022}
}

@inproceedings{li2021federatedcontrastive,
  title={Model-contrastive federated learning},
  author={Li, Qinbin and He, Bingsheng and Song, Dawn},
  booktitle={2021 IEEE/CVF Conference on Computer Vision and Pattern Recognition (CVPR)},
  pages={10708--10717},
  year={2021},
  organization={IEEE}
}

@inproceedings{li2020federated,
  title={Federated Optimization in Heterogeneous Networks},
  author={Tian Li and Anit Kumar Sahu and Ameet Talwalkar and Virginia Smith},
  booktitle={MLSys},
  year={2020}
}

@inproceedings{karimireddy2020scaffold,
  title={SCAFFOLD: Stochastic Controlled Averaging for Federated Learning},
  author={Sai Praneeth Karimireddy and Satyen Kale and Mehryar Mohri and Sashank J. Reddi and Sebastian Stich and Ananda Theertha Suresh},
  booktitle={ICML},
  year={2020}
}

@inproceedings{mcmahan2017communication,
  title={Communication-efficient learning of deep networks from decentralized data},
  author={McMahan, Brendan and Moore, Eider and Ramage, Daniel and Hampson, Seth and y Arcas, Blaise Aguera},
  booktitle={Artificial intelligence and statistics},
  pages={1273--1282},
  year={2017},
  organization={Pmlr}
}

@inproceedings{dong2021federated,
  title={Federated contrastive learning for decentralized unlabeled medical images},
  author={Dong, Nanqing and Voiculescu, Irina},
  booktitle={International Conference on Medical Image Computing and Computer-Assisted Intervention},
  pages={378--387},
  year={2021},
  organization={Springer}
}

@inproceedings{han2022fedx,
  title={Fedx: Unsupervised federated learning with cross knowledge distillation},
  author={Han, Sungwon and Park, Sungwon and Wu, Fangzhao and Kim, Sundong and Wu, Chuhan and Xie, Xing and Cha, Meeyoung},
  booktitle={European Conference on Computer Vision},
  pages={691--707},
  year={2022},
  organization={Springer}
}

@inproceedings{seo2024relaxed,
  title={Relaxed contrastive learning for federated learning},
  author={Seo, Seonguk and Kim, Jinkyu and Kim, Geeho and Han, Bohyung},
  booktitle={Proceedings of the IEEE/CVF Conference on Computer Vision and Pattern Recognition},
  pages={12279--12288},
  year={2024}
}

@inproceedings{louizosmutual,
  title={A Mutual Information Perspective on Federated Contrastive Learning},
  author={Louizos, Christos and Reisser, Matthias and Korzhenkov, Denis},
  booktitle={The Twelfth International Conference on Learning Representations},
  year={2024}
}

@article{ahmed2020active,
  title={Active Learning Based Federated Learning for Waste and Natural Disaster Image Classification},
  author={Ahmed, Lulwa and Ahmad, Kashif and Said, Naina and Qolomany, Basheer and Qadir, Junaid and Al-Fuqaha, Ala},
  journal={IEEE Access},
  volume={8},
  pages={208518--208531},
  year={2020},
  doi={10.1109/ACCESS.2020.3038676}
}

@inproceedings{chatzikokolakis2013broadening,
  title={Broadening the scope of differential privacy using metrics},
  author={Chatzikokolakis, Konstantinos and Andr{\'e}s, Miguel E and Bordenabe, Nicol{\'a}s Emilio and Palamidessi, Catuscia},
  booktitle={international symposium on privacy enhancing technologies symposium},
  pages={82--102},
  year={2013},
  organization={Springer}
}

@article{wang2014new,
  title={A new active learning method in Bayesian network},
  author={Wang, Jing and Shang, Xiaopu},
  journal={Applied Intelligence},
  volume={41},
  pages={1087--1103},
  year={2014},
  publisher={Springer}
}

@inproceedings{scheffer2001active,
  title={Active hidden Markov models for information extraction},
  author={Scheffer, Tobias and Decomain, Christian and Wrobel, Stefan},
  booktitle={International Conference on Industrial, Engineering and Other Applications of Applied Intelligent Systems},
  pages={309--318},
  year={2001},
  publisher={Springer}
}

@article{yehuda2022active,
  title={Active learning through a covering lens},
  author={Yehuda, Ofer and Dekel, Avihu and Hacohen, Guy and Weinshall, Daphna},
  journal={Advances in Neural Information Processing Systems},
  volume={35},
  pages={22354--22367},
  year={2022}
}

@techreport{krizhevsky2009learning,
  title={Learning Multiple Layers of Features from Tiny Images},
  author={Alex Krizhevsky and Geoffrey Hinton},
  institution={University of Toronto},
  year={2009}
}


\appendix
\section*{\centering Appendix}

\section{Experimental Setup and Methodology}
\label{sec:experimental_methodology}
{Abl}
We evaluate CIFAR-10, CIFAR-100~\citep{krizhevsky2009learning}, and ImageNet-100~\citep{deng2009imagenet} using TypiClust, MaxHerding, and ProbCover (see Section~\ref{sec:related}). Following the procedure described in Section~\ref{sec:method}, the SimCLR feature extractor is trained with FedAvg for 1,000 rounds, with one local epoch per round and batch size 256. Downstream classification uses a one-hidden-layer network with 256 hidden units trained with FedAvg.

We compare against a strengthened per-client baseline adapted from \citet{ono2025exploring}, which was shown to outperform other FAL baselines in the low-budget regime. In this baseline, clients learn separate feature spaces and perform active selection locally. Since these feature spaces are not aligned, standard FL aggregation is not directly applicable; we therefore train one downstream pipeline per client. At inference time, all client pipelines evaluate each test sample, and the final prediction is taken from the pipeline assigning the highest probability to its predicted class. We use this maximum-confidence aggregation under both IID and non-IID client partitions.  
This further strengthens the original ResNet-based protocol of Ono et al.~\cite{ono2025exploring}, as classifiers trained on self-supervised features have been shown to substantially outperform ResNet-based classifiers trained on raw images in the low-budget regime~\cite{hacohen2022active,yehuda2022active}.

In all experiments throughout the paper, CIFAR-10 is split across two clients, while CIFAR-100 and ImageNet-100 are split across four clients. Client partitions are generated as described in Appendix~\ref{app:split}. Selection hyperparameters are set according to the values specified in Appendix~\ref{app:hyperparameters}. All experiments are repeated over three random seeds (0-2), and we report the mean performance together with the standard error across seeds.

The experiments were run on a small local GPU cluster; each experiment used between 1 and 4 GPUs, with runtimes ranging from roughly one hour to one day depending on the dataset, method, and budget configuration.

\section{Client-Constrained Active Selection}
\label{app:client_constrained_selection}

We adapt ProbCover and TypiClust to global selection under
client-level annotation budgets.

\paragraph{Client-constrained ProbCover.}
Let $\mathcal{X}$ denote the embedding space and let
\[
X
=
\bigsqcup_{i=1}^{K} X_i
\subseteq \mathcal{X}
\]
denote the global unlabeled set, where $X_i$ is held by client $i$.
Let $P$ denote the underlying data distribution, $b$ the global
annotation budget, and $b_i$ the annotation budget of client $i$,
assuming $\sum_{i=1}^{K} b_i = b$.

Following ProbCover~\cite{yehuda2022active}, define
\[
\begin{aligned}
B_\delta(x)
&=
\left\{
x' \in \mathcal{X} : \|x-x'\|_2 \leq \delta
\right\},\\
C(L,\delta)
&=
\bigcup_{x\in L} B_\delta(x).
\end{aligned}
\]

\begin{definition}[Client-Constrained Max Probability Cover]
The client-constrained extension of Max Probability Cover is
\[
L^\star
\in
\arg\max_{\substack{
    L\subseteq X,\ |L|=b\\
    |L\cap X_i|\leq b_i,\ \forall i
}}
P\bigl(C(L,\delta)\bigr).
\]
\end{definition}

Using the empirical distribution on $X$, this becomes
\[
L^\star
\in
\arg\max_{\substack{
    L\subseteq X,\ |L|=b\\
    |L\cap X_i|\leq b_i,\ \forall i
}}
\left|
\bigcup_{x\in L}
\left(B_\delta(x)\cap X\right)
\right|.
\]
This is a typed, or colored, maximum-coverage problem, where each
candidate is associated with its originating client.

The coverage objective remains monotone submodular. The feasible sets
\[
\mathcal{I}
=
\left\{
L\subseteq X:
|L|\leq b,\ 
|L\cap X_i|\leq b_i\ \forall i
\right\}
\]
form a truncated partition matroid. Consequently, the standard greedy
algorithm achieves a $1/2$ approximation, compared with the
$(1-1/e)$ guarantee obtained by ProbCover under a single cardinality
constraint.

At each iteration, the server selects the feasible candidate with the
largest marginal coverage gain,
\[
\Delta(x\mid L)
=
\left|
\left(B_\delta(x)\cap X\right)
\setminus
\bigcup_{z\in L}\left(B_\delta(z)\cap X\right)
\right|.
\]
Candidates belonging to clients whose budgets have been exhausted are
excluded, and selection continues until the global budget $b$ is
reached.

\paragraph{Client-constrained TypiClust.}
TypiClust is applied to the aggregated embedding set, with cluster
representatives selected subject to the client-level annotation budgets.

\section{Ablation Study}
\label{sec:ablation}

Tables~\ref{tab:ablation_iid} and~\ref{tab:ablation_noniid_alpha} provide the absolute accuracies underlying the same-embedding coordination-gap results in Figure~\ref{fig:same_embedding_auc_gap}, for IID and non-IID data client distributions respectively. Centralized and per-client selection use the identical Phase-I embedding and downstream training protocol, differing only in query selection. Note that Table~\ref{tab:ablation_noniid_alpha} reports absolute accuracies as the Dirichlet parameter varies. The decreasing coordination benefit as $\alpha$ decreases corresponds to the trend summarized in Figure~\ref{fig:same_embedding_auc_gap}.

\begin{table*}[t]
\centering
\small
\begin{tabular}{lccccccccc}
\toprule
& \multicolumn{3}{c}{CIFAR-10} & \multicolumn{3}{c}{CIFAR-100} & \multicolumn{3}{c}{ImageNet-100} \\
\cmidrule(lr){2-4} \cmidrule(lr){5-7} \cmidrule(lr){8-10}
Method & 10 & 50 & 100 & 100 & 500 & 1000 & 100 & 500 & 1000 \\
\midrule
\multicolumn{10}{l}{\textit{ProbCover}} \\
Centralized & \textbf{49.5$_{\pm3.1}$} & \textbf{78.7$_{\pm0.5}$} & \textbf{84.4$_{\pm0.1}$} & \textbf{41.3$_{\pm0.1}$} & \textbf{48.6$_{\pm0.6}$} & \textbf{51.2$_{\pm0.6}$} & \textbf{40.2$_{\pm0.6}$} & \textbf{50.6$_{\pm0.4}$} & \textbf{54.9$_{\pm0.1}$} \\
CentEmb per-client & 38.1$_{\pm3.0}$ & 73.8$_{\pm1.1}$ & 79.1$_{\pm0.1}$ & 27.3$_{\pm0.4}$ & 45.2$_{\pm0.2}$ & 49.8$_{\pm0.4}$ & 23.4$_{\pm1.0}$ & 44.4$_{\pm0.1}$ & 50.2$_{\pm0.5}$ \\
\cdashline{1-10}
\multicolumn{10}{l}{\textit{TypiClust}} \\
Centralized & \textbf{32.6$_{\pm5.5}$} & \textbf{78.2$_{\pm0.7}$} & \textbf{82.2$_{\pm0.7}$} & \textbf{37.7$_{\pm0.5}$} & \textbf{46.3$_{\pm0.1}$} & \textbf{50.1$_{\pm0.1}$} & \textbf{38.9$_{\pm0.7}$} & \textbf{49.5$_{\pm0.5}$} & \textbf{53.8$_{\pm0.4}$} \\
CentEmb per-client & 31.8$_{\pm5.2}$ & 73.2$_{\pm0.8}$ & \textbf{82.2$_{\pm0.3}$} & 30.9$_{\pm0.4}$ & 44.6$_{\pm0.2}$ & 47.9$_{\pm0.2}$ & 29.6$_{\pm0.8}$ & 46.0$_{\pm0.3}$ & 51.1$_{\pm0.5}$ \\
\cdashline{1-10}
\multicolumn{10}{l}{\textit{MaxHerding}} \\
Centralized & \textbf{61.7$_{\pm2.6}$} & \textbf{83.1$_{\pm0.5}$} & \textbf{84.2$_{\pm0.2}$} & \textbf{36.2$_{\pm0.4}$} & \textbf{46.9$_{\pm0.2}$} & \textbf{50.2$_{\pm0.1}$} & \textbf{36.1$_{\pm0.8}$} & \textbf{49.1$_{\pm0.2}$} & \textbf{54.2$_{\pm0.5}$} \\
CentEmb per-client & 48.2$_{\pm2.4}$ & 79.1$_{\pm0.9}$ & 83.3$_{\pm0.9}$ & 25.0$_{\pm0.6}$ & 44.8$_{\pm0.3}$ & 48.7$_{\pm0.4}$ & 20.7$_{\pm0.7}$ & 44.9$_{\pm0.3}$ & 51.0$_{\pm0.2}$ \\
\bottomrule
\end{tabular}
\caption{Same-embedding selection-stage ablation (Figure~\ref{fig:same_embedding_iid_gap}, IID): test accuracy (\%) at 1, 5, and 10 labeled samples per class, comparing centralized selection to per-client selection on the identical shared embedding (CentEmb per-client). Mean $\pm$ standard error over 3 seeds. Bold marks the better of the pair per column, within each method block.}
\label{tab:ablation_iid}
\end{table*}

\begin{table*}[t]
\centering
\small
\begin{tabular}{lccccccccc}
\toprule
& \multicolumn{3}{c}{$\alpha=0.1$} & \multicolumn{3}{c}{$\alpha=1$} & \multicolumn{3}{c}{$\alpha=10$} \\
\cmidrule(lr){2-4} \cmidrule(lr){5-7} \cmidrule(lr){8-10}
Method & 100 & 500 & 1000 & 100 & 500 & 1000 & 100 & 500 & 1000 \\
\midrule
\multicolumn{10}{l}{\textit{ProbCover}} \\
Centralized & 31.5$_{\pm0.8}$ & 44.7$_{\pm0.8}$ & 48.3$_{\pm0.4}$ & \textbf{39.5$_{\pm0.1}$} & \textbf{47.4$_{\pm0.3}$} & \textbf{49.9$_{\pm0.2}$} & \textbf{41.2$_{\pm0.5}$} & \textbf{48.4$_{\pm0.4}$} & \textbf{50.9$_{\pm0.4}$} \\
CentEmb per-client & \textbf{34.7$_{\pm0.6}$} & \textbf{46.5$_{\pm0.4}$} & \textbf{49.6$_{\pm0.1}$} & 35.9$_{\pm0.8}$ & 46.0$_{\pm0.5}$ & 49.3$_{\pm0.4}$ & 30.6$_{\pm0.2}$ & 45.7$_{\pm0.6}$ & 49.6$_{\pm0.2}$ \\
\cdashline{1-10}
\multicolumn{10}{l}{\textit{TypiClust}} \\
Centralized & 28.7$_{\pm0.8}$ & 41.4$_{\pm0.5}$ & 46.0$_{\pm0.4}$ & \textbf{36.8$_{\pm0.7}$} & \textbf{45.0$_{\pm0.5}$} & \textbf{48.6$_{\pm0.4}$} & \textbf{37.5$_{\pm0.6}$} & \textbf{46.5$_{\pm0.3}$} & \textbf{50.1$_{\pm0.2}$} \\
CentEmb per-client & \textbf{32.8$_{\pm0.6}$} & \textbf{43.8$_{\pm0.5}$} & \textbf{47.6$_{\pm0.1}$} & 34.0$_{\pm1.0}$ & \textbf{45.0$_{\pm0.5}$} & 47.8$_{\pm0.4}$ & 31.9$_{\pm0.4}$ & 45.2$_{\pm0.6}$ & 48.7$_{\pm0.1}$ \\
\cdashline{1-10}
\multicolumn{10}{l}{\textit{MaxHerding}} \\
Centralized & 30.2$_{\pm1.5}$ & 43.1$_{\pm0.9}$ & 47.3$_{\pm0.2}$ & \textbf{35.8$_{\pm0.5}$} & \textbf{46.2$_{\pm0.3}$} & \textbf{49.2$_{\pm0.3}$} & \textbf{36.2$_{\pm0.4}$} & \textbf{46.6$_{\pm0.5}$} & \textbf{50.2$_{\pm0.3}$} \\
CentEmb per-client & \textbf{33.6$_{\pm0.8}$} & \textbf{45.1$_{\pm0.2}$} & \textbf{49.0$_{\pm0.4}$} & 30.5$_{\pm0.5}$ & 45.7$_{\pm0.4}$ & 48.6$_{\pm0.4}$ & 27.8$_{\pm0.0}$ & 44.7$_{\pm0.1}$ & 49.1$_{\pm0.3}$ \\
\bottomrule
\end{tabular}
\caption{Same-embedding selection-stage ablation (Figure~\ref{fig:same_embedding_noniid_gap}, non-IID, CIFAR-100): test accuracy (\%) at 1, 5, and 10 labeled samples per class, comparing centralized selection to per-client selection on the identical shared embedding (CentEmb per-client), across Dirichlet heterogeneity $\alpha \in \{0.1, 1, 10\}$. Mean $\pm$ standard error over 3 seeds. Bold marks the better of the pair per column, within each method block.}
\label{tab:ablation_noniid_alpha}
\end{table*}

\begin{table*}[t]
\centering
\small
\begin{tabular}{lccccccccc}
\toprule
& \multicolumn{3}{c}{TypiClust} & \multicolumn{3}{c}{ProbCover} & \multicolumn{3}{c}{MaxHerding} \\
\cmidrule(lr){2-4} \cmidrule(lr){5-7} \cmidrule(lr){8-10}
$\varepsilon$ & 100 & 500 & 1000 & 100 & 500 & 1000 & 100 & 500 & 1000 \\
\midrule
0.0 & 37.7$_{\pm0.5}$ & 46.3$_{\pm0.1}$ & 50.1$_{\pm0.1}$ & 41.3$_{\pm0.1}$ & 48.6$_{\pm0.6}$ & 51.2$_{\pm0.6}$ & 36.2$_{\pm0.4}$ & 46.9$_{\pm0.2}$ & 50.2$_{\pm0.1}$ \\
0.1 & 37.5$_{\pm0.1}$ & 46.7$_{\pm0.2}$ & 50.1$_{\pm0.1}$ & 41.7$_{\pm0.2}$ & 48.4$_{\pm0.4}$ & 50.8$_{\pm0.3}$ & 35.9$_{\pm0.1}$ & 46.9$_{\pm0.2}$ & 50.3$_{\pm0.4}$ \\
0.3 & 37.2$_{\pm0.2}$ & 47.0$_{\pm0.3}$ & 50.3$_{\pm0.1}$ & 39.5$_{\pm0.5}$ & 47.1$_{\pm0.4}$ & 50.6$_{\pm0.3}$ & 36.7$_{\pm0.5}$ & 47.2$_{\pm0.7}$ & 50.3$_{\pm0.1}$ \\
0.6 & 37.7$_{\pm0.4}$ & 46.2$_{\pm0.1}$ & 50.1$_{\pm0.2}$ & 23.0$_{\pm1.3}$ & 41.3$_{\pm0.6}$ & 47.5$_{\pm0.3}$ & 36.6$_{\pm0.4}$ & 46.8$_{\pm0.1}$ & 49.8$_{\pm0.2}$ \\
1.0 & 36.3$_{\pm0.5}$ & 45.4$_{\pm0.5}$ & 49.4$_{\pm0.4}$ & 24.2$_{\pm0.7}$ & 41.3$_{\pm0.5}$ & 47.1$_{\pm0.1}$ & 32.4$_{\pm0.6}$ & 44.1$_{\pm0.7}$ & 48.1$_{\pm0.3}$ \\
\cdashline{1-10}
CentEmb per-client & 30.9$_{\pm0.4}$ & 44.6$_{\pm0.2}$ & 47.9$_{\pm0.2}$ & 27.4$_{\pm0.4}$ & 45.2$_{\pm0.2}$ & 49.8$_{\pm0.4}$ & 25.0$_{\pm0.6}$ & 44.8$_{\pm0.3}$ & 48.7$_{\pm0.4}$ \\
\bottomrule
\end{tabular}
\caption{Test accuracy (\%) at 1, 5, and 10 labeled samples per class on CIFAR-100, for centralized active selection under increasing embedding noise $\varepsilon$ (rows), for each method (column groups). Mean $\pm$ standard error over 3 seeds. The bottom row is the noiseless aligned-embedding per-client baseline (CentEmb per-client).}
\label{tab:eps_sweep_absolute}
\end{table*}

\section{Embedding Perturbation}
\label{app:noise}
Let $\mathbf{X} \in \mathbb{R}^{N \times d}$ denote the matrix of embeddings, where each row $\mathbf{x}_i$ is normalized such that $\|\mathbf{x}_i\|_2 = 1$.  
We apply a stochastic perturbation that preserves unit norm while controlling the expected $\ell_2$ displacement.

For each embedding $\mathbf{x}_i$, we sample a Gaussian vector
\[
\mathbf{g} \sim \mathcal{N}(\mathbf{0}, I_d),
\]
and project it onto the tangent space of the unit sphere at $\mathbf{x}_i$:
\[
\mathbf{g}_\perp = \mathbf{g} - (\mathbf{g}^\top \mathbf{x}_i)\mathbf{x}_i,
\]
which ensures $\mathbf{g}_\perp^\top \mathbf{x}_i = 0$.  
We then scale the perturbation as
\[
\boldsymbol{\xi}_i = \frac{\sigma}{\sqrt{d-1}} \mathbf{g}_\perp,
\]
so that
\[
\mathbb{E}\|\boldsymbol{\xi}_i\|_2^2 = \sigma^2.
\]

The perturbed embedding is defined by
\[
\mathbf{x}_i' = \frac{\mathbf{x}_i + \boldsymbol{\xi}_i}{\|\mathbf{x}_i + \boldsymbol{\xi}_i\|_2},
\]
which guarantees $\|\mathbf{x}_i'\|_2 = 1$.

\paragraph{Choice of $\sigma$.}
The parameter $\sigma$ is chosen such that the expected displacement after normalization matches a target value $\epsilon$.  
Since $\boldsymbol{\xi}_i \perp \mathbf{x}_i$, we have
\[
\|\mathbf{x}_i + \boldsymbol{\xi}_i\|_2 = \sqrt{1 + \|\boldsymbol{\xi}_i\|_2^2}.
\]

In high dimensions, $\|\boldsymbol{\xi}_i\|_2^2$ concentrates sharply around its expectation $\sigma^2$, yielding the approximation
\[
\mathbf{x}_i^\top \mathbf{x}_i' 
\approx 
\frac{1}{\sqrt{1 + \sigma^2}}.
\]

The squared displacement is therefore
\[
\|\mathbf{x}_i' - \mathbf{x}_i\|_2^2 
= 2 - 2\, \mathbf{x}_i^\top \mathbf{x}_i'
\approx
2 - \frac{2}{\sqrt{1 + \sigma^2}}.
\]

Matching the expected displacement to $\epsilon^2$ gives
\[
\epsilon \approx \sqrt{2 - \frac{2}{\sqrt{1 + \sigma^2}}},
\]
which yields
\[
\sigma = \sqrt{\frac{1}{(1 - \epsilon^2/2)^2} - 1}.
\]

For $\epsilon < \sqrt{2}$, this mapping is bijective in $\sigma$, ensuring stable calibration.

\paragraph{Properties.}
This construction has three key properties:
(i) it preserves unit norm exactly,  
(ii) it induces an isotropic perturbation in the tangent space, and  
(iii) it provides explicit control over the expected displacement via $\epsilon$, with strong concentration in high dimension.

\subsection*{Empirical Evaluation}

We evaluate embedding obfuscation by adding noise before clients transmit embeddings to the server. Our goal is to determine how much perturbation can be introduced while preserving the gains of centralized selection. We compare noisy centralized selection against independent per-client selection in the aligned embedding space, following the protocol in Appendix~\ref{sec:ablation}. Results are shown in 
Table~\ref{tab:eps_sweep_absolute}, which reports the absolute accuracies underlying Figure~\ref{fig:noise_auc_vs_eps}.

\section{Client Split Generation Methodology}
\label{app:split}

We consider both IID and non-IID client partitions. In the IID setting, we randomly split the balanced dataset across clients, so that each client receives approximately the same number of samples and the label distribution is preserved across clients.

In the non-IID setting, each client is assigned data with a distinct label distribution, generated using a Dirichlet-based partitioning scheme. Let $C$ denote the number of classes and $K$ the number of clients. For each client $k \in \{1, \dots, K\}$, we sample a class-probability vector
\[
p_k \sim \mathrm{Dir}(\alpha \mathbf{1}),
\]
where $\mathbf{1} \in \mathbb{R}^C$ is the all-ones vector and $\alpha > 0$ is a concentration parameter.

Given these sampled distributions, the dataset is partitioned class-wise. For each class $c$, we collect all samples belonging to that class and distribute them among clients according to the probabilities $\{p_k[c]\}_{k=1}^K$. This assignment is performed while enforcing that each client receives approximately the same total number of samples, thereby preserving balanced dataset sizes across clients while inducing heterogeneous label distributions.

The parameter $\alpha$ controls the degree of heterogeneity: for $\alpha \ll 1$, the resulting distributions are highly skewed, leading to strongly non-IID client data; for $\alpha \approx 1$, the distributions are moderately heterogeneous; and for $\alpha \gg 1$, the class proportions concentrate around uniformity, yielding approximately IID client distributions.

\section{Derivation of the Redundancy Gap}
\label{app:coverage-gap-proofs}

This appendix derives Proposition~\ref{prop:heterogeneity-gap} and Corollaries~\ref{cor:disjoint-client} and~\ref{cor:iid-gap}, and relates the redundancy gap to the coordination gap.

\subsection{Expected Coverage}

Let
\[
S_{\mathrm{loc}}
=
\bigcup_{k=1}^K S_k^\star, \qquad L_K
=
\sum_{k=1}^K
\mathbb{E}[\Phi(S_k^\star)].
\]
For each cell $C_m$, define
\[
I_m
=
\mathbf{1}\{S_{\mathrm{loc}}\cap C_m\neq\emptyset\}.
\]
Then
\[
\Phi(S_{\mathrm{loc}})
=
\sum_{m=1}^M I_m.
\]
By linearity of expectation,
\begin{equation}
\label{eq:app-coverage-start}
\mathbb{E}[\Phi(S_{\mathrm{loc}})]
=
\sum_{m=1}^M
\Pr(S_{\mathrm{loc}}\cap C_m\neq\emptyset).
\end{equation}

A cell is not covered by $S_{\mathrm{loc}}$ if none of the clients covers it. Conditional independence gives
\[
\begin{aligned}
\Pr(S_{\mathrm{loc}}\cap C_m=\emptyset)
&=
\prod_{k=1}^K
\Pr(S_k^\star\cap C_m=\emptyset)\\
&=
\prod_{k=1}^K(1-a_{k,m}).
\end{aligned}
\]
Therefore,
\begin{equation}
\label{eq:app-expected-coverage}
\mathbb{E}[\Phi(S_{\mathrm{loc}})]
=
\sum_{m=1}^M
\left(
1-\prod_{k=1}^K(1-a_{k,m})
\right).
\end{equation}

\subsection{Local Coverage and Redundancy}

For each client,
\[
\Phi(S_k^\star)
=
\sum_{m=1}^M
\mathbf{1}\{S_k^\star\cap C_m\neq\emptyset\}.
\]
Hence,
\[
\begin{aligned}
\mathbb{E}[\Phi(S_k^\star)]
&=
\sum_{m=1}^M
\mathbb{E}
\left[
\mathbf{1}\{S_k^\star\cap C_m\neq\emptyset\}
\right]\\
&=
\sum_{m=1}^M
\Pr(S_k^\star\cap C_m\neq\emptyset)\\
&=
\sum_{m=1}^M a_{k,m}.
\end{aligned}
\]
Consequently,
\begin{equation}
\label{eq:app-LK}
L_K
=
\sum_{k=1}^K\sum_{m=1}^M a_{k,m}.
\end{equation}

The quantity $L_K$ counts coverage with multiplicity: a cell covered by multiple clients contributes once for each client. Define the expected cross-client redundancy as
\begin{equation}
\label{eq:app-redundancy-definition}
\bar{\Delta}_K^{\mathrm{red}}
=
L_K
-
\mathbb{E}[\Phi(S_{\mathrm{loc}})].
\end{equation}
Using Eqs.~\eqref{eq:app-expected-coverage} and \eqref{eq:app-LK},
\begin{equation}
\label{eq:app-redundancy}
\bar{\Delta}_K^{\mathrm{red}}
=
\sum_{m=1}^M
\left[
\sum_{k=1}^K a_{k,m}
-
\left(
1-\prod_{k=1}^K(1-a_{k,m})
\right)
\right].
\end{equation}

The expected coordination gap admits the exact decomposition
\begin{equation}
\label{eq:app-gap-decomposition}
\bar{\Delta}_K
=
\left(
\mathbb{E}[\Phi_B^\star]-L_K
\right)
+
\bar{\Delta}_K^{\mathrm{red}}.
\end{equation}
The first term measures the loss due to ineffective local coverage, while $\bar{\Delta}_K^{\mathrm{red}}$ measures redundancy across clients.

In the low-budget regime, a locally effective diversity-based selector is expected to cover approximately one new local cell with each query. If client $k$ has budget $b_k$, then
\[
\mathbb{E}[\Phi(S_k^\star)]
\approx
b_k.
\]
Since
\[
B
=
\sum_{k=1}^K b_k,
\]
this gives
\[
L_K
\approx
B.
\]
If the centralized selector also covers approximately one new cell per query, then
\[
\mathbb{E}[\Phi_B^\star]
\approx
B,
\]
and therefore
\[
L_K
\approx
\mathbb{E}[\Phi_B^\star].
\]

Under the idealized equality
\[
L_K
=
\mathbb{E}[\Phi_B^\star],
\]
Eq.~\eqref{eq:app-gap-decomposition} becomes
\begin{equation}
\label{eq:app-gap-redundancy}
\bar{\Delta}_K
=
\bar{\Delta}_K^{\mathrm{red}}
=
\sum_{m=1}^M
\left[
\sum_{k=1}^K a_{k,m}
-
\left(
1-\prod_{k=1}^K(1-a_{k,m})
\right)
\right].
\end{equation}

\subsection{Disjoint Accessible Cells}
\label{app:propo1-proof}

We first show that pairwise-disjoint accessible cell sets yield zero redundancy.

\begin{proof}[Proof of Proposition~\ref{prop:heterogeneity-gap}, Part 1]
Assume
\[
A_i\cap A_j=\emptyset
\qquad
\forall i\neq j.
\]
For every cell $C_m$, at most one client has $a_{k,m}>0$.

If all $a_{k,m}$ are zero, then
\[
1-\prod_{k=1}^K(1-a_{k,m})
=
0
=
\sum_{k=1}^K a_{k,m}.
\]

Otherwise, let $j$ be the unique client satisfying $a_{j,m}>0$. Then
\[
a_{k,m}=0
\qquad
\forall k\neq j,
\]
and
\[
\begin{aligned}
1-\prod_{k=1}^K(1-a_{k,m})
&=
1-(1-a_{j,m})
\prod_{k\neq j}(1-a_{k,m})\\
&=
1-(1-a_{j,m})\\
&=
a_{j,m}\\
&=
\sum_{k=1}^K a_{k,m}.
\end{aligned}
\]
Thus, every cell contributes zero to Eq.~\eqref{eq:app-redundancy}, implying
\[
\bar{\Delta}_K^{\mathrm{red}}=0.
\]
\end{proof}

\subsection{Shared Selection Mass}
We next show that increasing a client's selection probability for a cell can only increase that cell's redundancy contribution.

For a fixed cell $C_m$, define
\[
\Delta_m
=
\sum_{k=1}^K a_{k,m}
-
\left(
1-\prod_{k=1}^K(1-a_{k,m})
\right).
\]

\begin{proof}[Proof of Proposition~\ref{prop:heterogeneity-gap}, Part 2]
Differentiating with respect to $a_{j,m}$ gives
\[
\begin{aligned}
\frac{\partial\Delta_m}{\partial a_{j,m}}
&=
1-
\frac{\partial}{\partial a_{j,m}}
\left(
1-\prod_{k=1}^K(1-a_{k,m})
\right)\\
&=
1-\prod_{k\neq j}(1-a_{k,m}).
\end{aligned}
\]
Since $a_{k,m}\in[0,1]$,
\[
0
\leq
\prod_{k\neq j}(1-a_{k,m})
\leq
1,
\]
and therefore
\[
\frac{\partial\Delta_m}{\partial a_{j,m}}
\geq
0.
\]

Equality holds if and only if
\[
a_{k,m}=0
\qquad
\forall k\neq j.
\]
Thus, the derivative is strictly positive exactly when another client selects from $C_m$ with positive probability.
\end{proof}

\subsection{Marginal Effect of an Additional Client}
Finally, we characterize the additional redundancy introduced by a new client. Define
\[
c_m^{(K)}
=
1-\prod_{k=1}^K(1-a_{k,m}),
\]
the probability that at least one of the first $K$ clients selects from $C_m$.

\begin{proof}[Proof of Proposition~\ref{prop:heterogeneity-gap}, Part 3]
From Eq.~\eqref{eq:app-redundancy},
\[
\bar{\Delta}_K^{\mathrm{red}}
=
\sum_{m=1}^M
\left[
\sum_{k=1}^K a_{k,m}
-
\left(
1-\prod_{k=1}^K(1-a_{k,m})
\right)
\right].
\]
Therefore,
\[
\begin{aligned}
\bar{\Delta}_{K+1}^{\mathrm{red}}-\bar{\Delta}_K^{\mathrm{red}}
&=
\sum_{m=1}^M
\Bigl[
a_{K+1,m}\\
&\qquad
+\prod_{k=1}^{K+1}(1-a_{k,m})
-\prod_{k=1}^{K}(1-a_{k,m})
\Bigr].
\end{aligned}
\]

For each cell,
\[
\prod_{k=1}^{K+1}(1-a_{k,m})
=
(1-a_{K+1,m})
\prod_{k=1}^{K}(1-a_{k,m}).
\]
Hence,
\[
\begin{aligned}
&
\left(
1-\prod_{k=1}^{K+1}(1-a_{k,m})
\right)
-
\left(
1-\prod_{k=1}^{K}(1-a_{k,m})
\right)\\
&=
a_{K+1,m}
\prod_{k=1}^{K}(1-a_{k,m}).
\end{aligned}
\]
Substituting this identity yields
\[
\begin{aligned}
\bar{\Delta}_{K+1}^{\mathrm{red}}-\bar{\Delta}_K^{\mathrm{red}}
&=
\sum_{m=1}^M
a_{K+1,m}
\left[
1-\prod_{k=1}^{K}(1-a_{k,m})
\right]\\
&=
\sum_{m=1}^M
a_{K+1,m}c_m^{(K)}.
\end{aligned}
\]
\end{proof}

We next show that adding a client whose accessible cells are disjoint from those of the existing clients introduces no additional redundancy.

\begin{proof}[Proof of Corollary~\ref{cor:disjoint-client}]
If
\[
A_{K+1}\cap\bigcup_{k=1}^K A_k=\emptyset,
\]
then
\[
a_{K+1,m}>0
\quad\Longrightarrow\quad
c_m^{(K)}=0.
\]
Therefore,
\[
\bar{\Delta}_{K+1}^{\mathrm{red}}
-
\bar{\Delta}_K^{\mathrm{red}}
=
\sum_{m=1}^M a_{K+1,m}c_m^{(K)}
=
0.
\]
\end{proof}

\subsection{IID Clients}
We finally show that, under identical cell-level selection probabilities, the redundancy gap is nondecreasing in the number of clients.

\begin{proof}[Proof of Corollary~\ref{cor:iid-gap}]
Under the IID assumption,
\[
a_{k,m}
=
a_m
\]
for every client $k$. Substituting into Eq.~\eqref{eq:app-redundancy} gives
\[
\bar{\Delta}_K^{\mathrm{red}}
=
\sum_{m=1}^M
\left[
K a_m
-
\left(
1-(1-a_m)^K
\right)
\right].
\]

For a fixed cell, define
\[
f_m(K)
=
K a_m
-
\left(
1-(1-a_m)^K
\right).
\]
Its finite difference is
\[
\begin{aligned}
f_m(K+1)-f_m(K)
&=
a_m-a_m(1-a_m)^K\\
&=
a_m
\left[
1-(1-a_m)^K
\right]\\
&\geq
0.
\end{aligned}
\]
Summing over cells gives
\[
\bar{\Delta}_{K+1}^{\mathrm{red}}
-
\bar{\Delta}_K^{\mathrm{red}}
=
\sum_{m=1}^M
a_m
\left[
1-(1-a_m)^K
\right]
\geq
0.
\]
Thus, under fixed per-client selection behavior, the expected redundancy gap is nondecreasing in the number of IID clients.
\end{proof}

\section{Hyperparameters}
\label{app:hyperparameters}

Table~\ref{tab:selection_hyperparameters} lists the selection hyperparameters used in our experiments. ProbCover radii follow the dataset-specific values recommended by \citet{yehuda2022active}; MaxHerding uses the fixed Gaussian-kernel lengthscale $\sigma=1$ from \citet{bae2024maxherding}.

\begin{table}[h]
\small
\centering
\begin{tabular}{@{}lll@{}}
\toprule
\textbf{Dataset} &
\shortstack{\textbf{ProbCover}\\\textbf{radius $\delta$}} &
\shortstack{\textbf{MaxHerding}\\\textbf{lengthscale $\sigma$}} \\ \midrule
\textbf{CIFAR-10} & $0.75$ & $1.0$ \\
\textbf{CIFAR-100} & $0.65$ & $1.0$ \\
\textbf{ImageNet-100} & $0.55$ & $1.0$ \\
\bottomrule
\end{tabular}
\caption{Selection hyperparameters used in our experiments.}
\label{tab:selection_hyperparameters}
\end{table}

\end{document}